\documentclass{article}

\usepackage{microtype}
\usepackage{graphicx}
\usepackage{subcaption}
\usepackage{booktabs} % for professional tables

 \usepackage[preprint]{icml2026}

\usepackage{amsmath}
\usepackage{amssymb}
\usepackage{mathtools}
\usepackage{amsthm}
\usepackage{hyperref}

\usepackage[capitalize,noabbrev]{cleveref}

\theoremstyle{plain}
\newtheorem{theorem}{Theorem}[section]

\newtheorem{lemma}[theorem]{Lemma}

\theoremstyle{definition}
\newtheorem{definition}[theorem]{Definition}
\newtheorem{assumption}[theorem]{Assumption}
\theoremstyle{remark}
\newtheorem{remark}[theorem]{Remark}

\usepackage[textsize=tiny]{todonotes}

\icmltitlerunning{Capabilities and Fundamental Limits of Latent Chain-of-Thought}

\begin{document}

\twocolumn[
  \icmltitle{Capabilities and Fundamental Limits of Latent Chain-of-Thought}



  \icmlsetsymbol{equal}{*}

  \begin{icmlauthorlist}
    \icmlauthor{Jiaxuan Zou}{equal,xjtu}
    \icmlauthor{Yaozhong Xiong}{equal,ruc}
    \icmlauthor{Yong Liu}{ruc}
  \end{icmlauthorlist}

  \icmlaffiliation{xjtu}{Department of Mathematics and Statistics, Xi'an Jiaotong University, Xi'an 710049, Shaanxi Province, China}
  \icmlaffiliation{ruc}{Department of Artificial Intelligence, Gaoling School of Artificial Intelligence, Renmin University of China, Beijing 100872, China}

  \icmlcorrespondingauthor{Yong Liu}{liuyonggsai@ruc.edu.cn}

  \icmlkeywords{Machine Learning, ICML}

  \vskip 0.3in
]



\printAffiliationsAndNotice{}  % no special notice (required even if empty)

\begin{abstract}
Latent Chain-of-Thought (Latent CoT) models promise efficient reasoning via continuous representations, yet exhibit puzzling performance inconsistencies: excelling at exploration (ProsQA: 97.0\%) but failing at computation (GSM8K: 34.1\%). We reveal that this \textbf{trade-off} is governed by decisional certainty. Our contributions are threefold: (1) We theoretically characterize the fundamental \textit{Exploration-Execution Trade-off}, proving that high certainty enables precise execution but inhibits exploration, while low certainty facilitates search but causes error accumulation. (2) We introduce the \textbf{Symbolic Index}—quantifying decisional commitment—as the core mechanism governing this trade-off and establish its causal relationship with both execution stability and exploration capability. (3) We prove that curriculum learning is theoretically necessary, as direct training provably fails due to distributional mismatch. Our framework shifts the design paradigm from binary architectural choices toward adaptive systems that dynamically regulate decisional certainty based on task demands.
\end{abstract}

\section{Introduction}

The advent of Large Language Models (LLMs) \cite{Guo2025, OpenAI2025, DeepSeek-AI2025} has marked a paradigm shift in artificial intelligence, fundamentally transforming the landscape of complex reasoning tasks \cite{Liu2025a}. To date, the dominant methodology for eliciting reasoning capabilities has been the explicit Chain-of-Thought (CoT) \cite{Wei2022, Sun2023}. By forcing models to verbalize intermediate steps as discrete tokens, CoT imposes a structured, human-interpretable logic flow that significantly improves performance. However, this explicit verbalization comes at a steep cost: it suffers from inherent inefficiency due to excessive sequence lengths and the computational burden of autoregressive bottlenecks \cite{Hong2025, Yue2025a}. To address these scalability limitations, recent research has aggressively pivoted toward implicit reasoning, often termed Latent CoT \cite{Ye2025a}. In this paradigm, the model operates within a continuous internal state space, processing information through "silent" vector transformations without generating intermediate tokens. By leveraging mechanisms such as token-level manipulation \cite{Tack2025, Sun2025} and trajectory optimization \cite{Cheng2024, Hao2024}, these models promise to reduce computational complexity and enable more diverse, high-dimensional reasoning paths that are not constrained by the vocabulary of natural language \cite{Hao2024, Zhang2025a}.

Despite these architectural innovations, a stark and puzzling dichotomy has emerged in empirical performance. As presented in Table~\ref{tab:main_results}, Explicit CoT and Latent CoT exhibit complementary—and mutually exclusive—failure modes. Explicit CoT excels at tasks requiring precise symbolic manipulation and rigorous logic execution, such as the arithmetic reasoning in GSM8K (42.9\% accuracy). However, it struggles significantly with tasks necessitating flexible exploration or strategic planning, such as ProsQA (77.5\%), often committing too early to a rigid path. Conversely, Latent CoT achieves superior performance on exploratory tasks (97.0\%), leveraging its continuous representation to traverse a broader search space. Yet, this flexibility comes with a severe penalty: Latent CoT fails catastrophically on precision-heavy computational tasks, where exact state maintenance is critical. Furthermore, Latent CoT exhibits extreme training fragility; as shown in our ablation results, removing curriculum learning causes performance to collapse (e.g., from 99.8\% to 52.4\% on ProntoQA), suggesting a fundamental instability in how these models learn to reason.

\begin{table*}[t]
\caption{Performance of discrete (CoT) and latent (Coconut) reasoning models reveals a puzzling trade-off. CoT excels at precise computation (GSM8K) but struggles with flexible reasoning (ProsQA). Conversely, Latent CoT thrives on flexibility but fails at computation and is highly sensitive to its training curriculum. *Results from \cite{hao2024traininglargelanguagemodels}.}
\label{tab:main_results}
\centering
\begin{tabular}{@{}lcccccc@{}}
\toprule
\multicolumn{1}{c}{Method} & \multicolumn{2}{c}{GSM8K} & \multicolumn{2}{c}{ProntoQA} & \multicolumn{2}{c}{ProsQA} \\
\cmidrule(lr){2-3} \cmidrule(lr){4-5} \cmidrule(lr){6-7}
& Acc. (\%) & \# Tokens & Acc. (\%) & \# Tokens & Acc. (\%) & \# Tokens \\
\midrule
CoT & $42.9 \pm 0.2$ & $25.0$ & $98.8 \pm 0.8$ & $92.5$ & $77.5 \pm 1.9$ & $49.4$ \\
No-CoT & $16.5 \pm 0.5$ & $2.2$ & $93.8 \pm 0.7$ & $3.0$ & $76.7 \pm 1.0$ & $8.2$ \\
\midrule
COCONUT & $34.1 \pm 1.5$ & $8.2$ & $99.8 \pm 0.2$ & $9.0$ & $97.0 \pm 0.3$ & $14.2$ \\
\quad - w/o curriculum & $14.4 \pm 0.8$ & $8.2$ & $52.4 \pm 0.4$ & $9.0$ & $76.1 \pm 0.2$ & $14.2$ \\
\bottomrule
\end{tabular}
\end{table*}

These empirical observations highlight a critical gap in our theoretical understanding. Currently, the field lacks a unified framework to explain why high performance in exploration necessitates a sacrifice in execution precision, and why implicit models are intrinsically difficult to train. Existing literature largely treats these architectures as binary choices—discrete versus continuous—without characterizing the underlying decision-theoretic mechanism that governs their respective strengths and weaknesses. Consequently, current architectural designs rely on heuristics rather than principled regulation of the model's internal reasoning process, leading to systems that are either precise but rigid, or flexible but hallucinatory.

In this work, we propose that this trade-off is not an arbitrary artifact of architecture, but a fundamental consequence of \textit{decisional certainty}. We introduce the \textbf{Symbolic Index}—a quantifiable measure of a model's commitment to a specific reasoning path—as the core mechanism unifying these phenomena. Our key insight is that Explicit CoT operates in a \textit{high-certainty regime}. This regime guarantees computational fidelity through an error-correcting discretization process (the selection of specific tokens), but it collapses exploration by committing too early to a single trajectory. In contrast, Latent CoT operates in a \textit{low-certainty regime}, maintaining a superposition of multiple reasoning paths. While this superposition enables robust exploration of complex solution spaces, it suffers from noise accumulation that eventually destroys symbolic precision, as there is no discrete quantization step to reset the internal state.

Our contributions are threefold. First, we theoretically characterize the fundamental exploration-execution trade-off. We prove that the high decisional certainty of CoT leads to vanishing exploration, while the low certainty of Latent CoT amplifies sub-decisional noise, serving as the direct cause of failure on arithmetic tasks. Second, we formalize the Symbolic Index as a regulatory metric, providing concrete design principles for next-generation architectures that can dynamically transition between exploration and execution. Third, we resolve the training stability puzzle by proving that curriculum learning is theoretically necessary for Latent CoT. We demonstrate that without a curriculum, training is guaranteed to fail due to a distributional mismatch between the model's self-generated latent states and valid reasoning trajectories—a gap that the curriculum provably bridges.
\section{Related Work}
    
Unlike explicit CoT, which is constrained to a single discrete path, Latent CoT leverages continuous latent spaces to achieve exploratory diversification. Research in this area enables models to explore multiple reasoning trajectories in parallel, enhancing robustness on complex problems. Techniques include sampling latent trajectories \cite{Chen2024a}, using probabilistically weighted concepts \cite{Zhang2025c}, and injecting continuous tokens to enrich the search space \cite{Xu2025b, Xu2025a, Gozeten2025}. These empirical successes demonstrate the value of exploratory capabilities. We formalize this intuition by proving that exploration capability stems from maintaining low decisional certainty—a low Symbolic Index—and by establishing formal bounds on the KL divergence from ideal uniform exploration priors (Theorem~\ref{thm:latent_cot_exploration_guarantee}).
    
A distinct line of research focuses on trajectory-level latent optimization, compressing entire explicit reasoning chains into continuous representations. Early works aimed to maintain semantic fidelity by anchoring latent states to explicit steps \cite{Cheng2024, Liu2024c, Shen2025b}. More recent efforts improve efficiency through dynamic compression or adaptive control \cite{Zhang2025a, Ma2025, Tan2025, Wang2025a}. A particularly important strategy is progressive internalization, where explicit steps are gradually replaced by latent thoughts via curriculum learning \cite{Deng2024, Hao2024, Shen2025a} or internal iterations \cite{Zeng2025, Ruan2025}. These methods, especially those using a curriculum, provide the empirical motivation for our analysis. Our work complements these empirical findings by establishing that progressive training is not merely beneficial but theoretically necessary to overcome a fundamental distributional mismatch (Theorem~\ref{thm:fail_without_curriculum}) and ensure convergence (Theorem~\ref{thm:success_with_curriculum}, Section \ref{sec:training_dynamics}).
    
Complementing these architectural advances, a parallel stream of research investigates the fine-grained control and analysis of internal states. Signal-guided control methods insert specialized, non-text-producing tokens to steer internal reasoning \cite{Herel2024, Goyal2024, Zelikman2024, Pfau2024, Wang2024c}. Simultaneously, internal state analysis uses techniques like probing or distillation to find mechanistic evidence of implicit reasoning, such as discovering encoded reasoning trees in attention patterns \cite{Deng2023, Hou2023, Wang2025c, Yu2024b}. Both approaches highlight the centrality of internal states. We unify these mechanistic insights by proving that variance in internal certainty is the root cause of the exploration-execution trade-off (Theorems~\ref{thm:symbolic_stability_revised} and \ref{thm:explore_exploit_tradeoff_revised}, Section \ref{sec:unifying_perspective}). Our Symbolic Index provides a computable, quantifiable metric for these otherwise opaque internal states, offering a principled evaluative criterion for future architectural innovations.

\section{Preliminaries}

We consider supervised learning for reasoning tasks, where the input is a random variable $X$ from which instances $x$ are drawn. Each input $x$ and answer $y$ are connected by a chain-of-thought $S = (s_1, \dots, s_M)$. We compare two paradigms for modeling this process: Chain-of-Thought (CoT) and Latent CoT, the latter trained via the Coconut curriculum.
\paragraph{Chain-of-Thought (CoT)}
CoT models autoregressively generate discrete reasoning steps:
$$p_\theta(S|x) = \prod_{k=1}^M p_\theta(s_k | x, S^{(1 \dots k-1)}).$$
At inference, token selection commits to a single path—enabling precise computation but risking early missteps \cite{2025_ICLR_CoTFormer_CoTFormer=A-Chain-of-Thought-Driven-Architecture-with-Budged-Adaptive-Computation-Cost-at-Inference, 2025_arXiv_RELAY_Enhancing-Auto-regressive-Chain-of-Thought-through-Loop-Aligned-Reasoning, 2024_arXiv_LLaVA-CoT_LLaVA-o1=Let-Vision-Language-Models-Reason-Step-by-step, 2025_arxiv_When-Chain-of-Thought-is-Necessary-Language-Models-Struggle-to-Evade-Monitors}.

\paragraph{Latent Chain-of-Thought (Latent CoT)}
Latent CoT iterates in continuous space, producing latent states $H = (h_1, \dots, h_M)$, $h_k \in \mathbb{R}^d$:
$$h_k = f_\theta(h_{k-1}).$$
Each $h_k$ encodes multiple potential reasoning paths, supporting exploration but accumulating noise over steps \cite{2025_ICML_LTMs_Scalable-Language-Models-with-Posterior-Inference-of-Latent-Thought-Vectors, 2025_arXiv_Reasoning-by-Superposition=A-Theoretical-Perspective-on-Chain-of-Continuous-Thought, 2025_ICML_Token-Assorted_Token-Assorted-Mixing-Latent-and-Text-Tokens-for-Improved-Language-Model-Reasoning, 2025_arXiv_Beyond-Words_Beyond-Words=A-Latent-Memory-Approach-to-Internal-Reasoning-in-LLMs}.

\paragraph{Coconut Training \cite{hao2024traininglargelanguagemodels}} \label{coconut_training}
Since direct supervision on $H$ is infeasible, Coconut uses a curriculum. At stage $k$, it learns to compress the prefix $S^{(1\dots k)}$ into a latent state $h_k$ to predict the suffix $S^{(k+1\dots M)}$:
$$\mathcal{L}_{\text{Coconut}}^{(k)}(\theta) = \mathbb{E}_{p(x,S)} [-\log p_\theta(S^{(k+1 \dots M)}|h_k, x)].
$$
The curriculum progresses from $k=0$ (pure CoT) to $k=M$ (full latent reasoning), gradually internalizing symbolic steps that bridge the discrete and continuous CoT.

\section{Theoretical Analysis}

This section provides a deep theoretical analysis of the fundamental differences and trade-offs between CoT and Latent CoT. We begin by revealing the mathematical underpinnings of the Coconut training objective through the lens of Information Bottleneck theory. Subsequently, we directly analyze model performance along two key dimensions: \textbf{planning and exploration capability} versus \textbf{computational execution precision}. Finally, we identify \textbf{decisional certainty}—quantified by the Symbolic Index—as the core mechanism regulating this trade-off, and theoretically establish the necessity of curriculum learning for effective training of Latent CoT models.

\subsection{Duality with the Information Bottleneck}
\label{cib_dual}
Directly analyzing the behavior of the Coconut-trained model is challenging due to its staged, implicit compression objective. To overcome this, we show that the Coconut curriculum is \textit{mathematically equivalent} to solving a well-studied information-theoretic objective—the Conditional Information Bottleneck \cite{tishby2015deeplearninginformationbottleneck}. This equivalence allows us to leverage established tools from information theory to rigorously characterize the model's properties.

This duality can be understood intuitively: the Coconut objective at stage $k$ poses an information compression problem. The model must compress the past chain-of-thought, $S^{(1\dots k)}$, into a single latent vector, $h_k$, with the sole purpose of maximizing its utility for predicting the future chain, $S^{(k+1\dots M)}$. This is precisely the problem addressed by the Information Bottleneck (IB) principle: finding a compressed "bottleneck" representation of a source (the past) that maximally preserves information about a target (the future). The Coconut training process thus implicitly learns an optimal information compressor. Theorem~\ref{thm:coconut_cib_duality} formalizes this intuition.

\begin{theorem}[Coconut-CIB Duality]
\label{thm:coconut_cib_duality}
Under the constraint of any finite model capacity, the optimization objective of the Coconut curriculum (Preliminaries~\ref{coconut_training}) at stage $k$ can be rigorously reformulated as a constrained optimization problem. Its Lagrangian dual is precisely the Conditional Information Bottleneck (CIB) problem. Specifically:

\textbf{Primal Problem:} The optimization process of Coconut is equivalent to solving:
\begin{align}
    & \min_{p(h_k \mid X)} H(S^{(k+1\dots M)} \mid h_k, X) \\ & \text{s.t.} \quad I(h_k;S^{(1\dots k)} \mid X) \le R \label{eq:primal} \\
\textbf{Dual:} \ & \min_{p(h_k \mid X)} \Big[ I(h_k;S^{(1\dots k)} \mid X) \nonumber \\
    & \phantom{\min_{p(h_k \mid X)} \Big[} - \beta(k) I(h_k;S^{(k+1\dots M)} \mid X) \Big] \label{eq:dual}
\end{align}
where $\beta(k) > 0$ ideally satisfies $\beta(k) \sim \frac{k}{M-k}$.
\end{theorem}

\paragraph{Interpretation of $\beta(k)$.} The trade-off parameter $\beta(k) \sim \frac{k}{M-k}$ reflects the curriculum's shifting emphasis across stages. As $k$ increases, the model compresses more past information while maintaining predictive power over a shorter future sequence, dynamically balancing compression efficiency and predictive utility.This dynamic weighting ensures optimal balance at each stage; we formalize the training necessity in §\ref{sec:training_dynamics}.

\textit{The main challenge is to connect Coconut's operational training loss to a formal information-theoretic objective. The proof achieves this by framing the optimization as a constrained problem and applying Lagrangian duality, revealing that Coconut implicitly solves the Conditional Information Bottleneck. Proof in Appendix~\ref{app:proof_thm_coconut_cib_duality}.}
\subsection{Latent CoT excels in exploration}
\label{exploration_abi}
We model reasoning as traversal over a decision DAG $Q := (G, v_{\text{start}}, V_{\text{target}})$. At any node $v$, let $N_{\text{valid}}(v)$ be the set of valid next steps. An ideal explorer uses a uniform prior $q_{\text{PR}}(u \mid v) = \frac{1}{|N_{\text{valid}}(v)|} \mathbb{I}[u \in N_{\text{valid}}(v)]$. We measure deviation from this ideal via $D_{\text{KL}}(q_{\text{PR}} \| p)$, where $p$ is the model’s next-step distribution: low KL divergence implies broad exploration; high KL divergence indicates overconfidence and premature commitment.

Our analysis first focuses on the intrinsic properties of a \textbf{single, coherent reasoning path}, which constitutes the fundamental building block of the CoT paradigm. Understanding the generative process of one such trajectory is essential for dissecting the core mechanism of execution. To formalize this, we characterize the high-certainty behavior exhibited by CoT during successful reasoning. This is not an arbitrary modeling choice but an inherent consequence of CoT's training via teacher-forcing on deterministic reasoning chains, which naturally induces sharply peaked distributions.

\begin{assumption}[$\kappa$-Concentrated Distribution for CoT]
\label{assump:cot_concentration}
At a decision point with $B$ options, we model the generative distribution of a single CoT step, $p_{\text{CoT}}$, as being drawn from a Dirichlet prior with a large concentration parameter $\kappa = \sum_i \alpha_i$. A large $\kappa$ yields a sharply peaked distribution, reflecting the high decisional certainty required for deterministic, high-fidelity execution.
\end{assumption}

Under this assumption, CoT inevitably collapses its distribution at each step:

\begin{theorem}[Exploration Deficiency of CoT]
\label{thm:cot_exploration_deficiency}
As $\kappa \to \infty$, the entropy $H(p_{\text{CoT}}) \to 0$, and
\begin{equation*}
\begin{split}
    D_{\text{KL}}(q_{\text{PR}} \| p_{\text{CoT}}) &= \frac{B-1}{B}\log\kappa - \log B \\
    &\quad - \frac{(B-1)\log c}{B} + \mathcal{O}\left(\frac{1}{\kappa}\right),
\end{split}
\end{equation*}
for a constant $c > 0$, implying that $D_{\text{KL}} \to \infty$ as certainty $\kappa \to \infty$.
\end{theorem}
\textit{Proof in Appendix~\ref{app:proof_thm_cot_exploration_deficiency}.} 
\begin{remark}[Sampling-Based Methods]
    Our analysis governs the generation of a single reasoning chain. While ensemble methods like Self-Consistency \cite{wang2023selfconsistencyimproveschainthought} introduce exploration by sampling multiple chains, the computational integrity of each \textit{individual} chain still hinges on the high-certainty commitments that define CoT's execution-focused nature. Therefore, our core results remain fundamental.
\end{remark}

This quantifies CoT's poor performance on exploratory tasks like ProsQA (77.5\% vs. Latent CoT's 97.0\% in Table~\ref{tab:main_results}): the divergence $D_{\text{KL}} \to \infty$ formalizes that CoT's distribution becomes arbitrarily distant from the uniform exploration prior.

In contrast, Latent CoT's training objective—dual to the Conditional Information Bottleneck (Theorem~\ref{thm:coconut_cib_duality})—imposes implicit regularization against overconfidence, ensuring sustained exploration:

\begin{theorem}[Exploration Capability Guarantee of Latent CoT]
\label{thm:latent_cot_exploration_guarantee}
There exist constants $\delta \in (0,1)$ and finite $c$ such that
\[
D_{\text{KL}}(q_{\text{PR}} \| p_{\text{Coconut}}) \leq -\frac{1}{2}\log\delta - c.
\]
\end{theorem}

\textit{Proof in Appendix~\ref{app:proof_thm_latent_cot_exploration_guarantee}.} This finite upper bound on $D_{\text{KL}}$ ensures that Latent CoT maintains a non-degenerate distribution over options, contrasting sharply with CoT's unbounded divergence (Theorem~\ref{thm:cot_exploration_deficiency}). This guarantees robust exploration, complementing CoT's strength in execution, which we analyze next.

\subsection{Latent CoT's Fragility in Symbolic Computation}
\label{exploitation_abli}

While Latent CoT’s continuous state space supports robust exploration (Section \ref{exploration_abi}), it is a liability for tasks requiring high-fidelity symbolic reasoning, such as GSM8K. Unlike CoT's discrete symbolic operations, Latent CoT's continuous state representations are inherently vulnerable to noise accumulation, which undermines precise step-by-step computation.

To formalize this, we analyze the effect of small internal errors, termed \textbf{sub-decisional perturbations}—noise that corrupts a model's internal state but is insufficient to alter the immediate output.

\begin{definition}[Sub-decisional Perturbation]
\label{def:sub_decisional_noise}
Let $l_k = l_k^* + \epsilon_k$ be the logit vector at step $k$. A perturbation $\epsilon_k$ is sub-decisional if
$$ \operatorname{argmax}(l_k^* + \epsilon_k) = \operatorname{argmax}(l_k^*). $$
\end{definition}

This concept highlights a key architectural divergence. CoT models are inherently robust to such perturbations due to their discrete generation process.

\begin{theorem}[Symbolic Integrity of CoT]
\label{thm:cot_final_state}
Let $S^* = (s_1^*, \dots, s_M^*)$ be the noise-free CoT trajectory and $\hat{S}$ the trajectory under sub-decisional perturbations. Then
$$ \mathbb{P}[\hat{s}_M \neq s_M^*] = 0. $$
\end{theorem}
\textit{Proof in Appendix~\ref{app:proof_thm_cot_final_state}.}

Theorem~\ref{thm:cot_final_state} formalizes CoT's resilience to sub-decisional perturbations. The underlying mechanism is a \textbf{discretization-reset} process: \emph{at each step, the argmax operation acts as an error-correcting filter}, projecting the continuous hidden state to a discrete token and thereby discarding any sub-decisional noise. This reset ensures that subsequent reasoning begins from a clean symbolic representation, preventing error propagation across the reasoning chain.

In stark contrast, Latent CoT lacks this discretization-reset mechanism. Its continuous state $h_k$ is passed directly to the next step without quantization, carrying forward any perturbation. For a model with transition function $f_\theta$ (Lipschitz constant $L_F$) and i.i.d. noise $\epsilon_h^{(k)} \sim \mathcal{N}(0, \sigma_h^2 I_d)$, this error propagation is quantifiable:

\begin{theorem}[Compounding Error in Latent Computation]
\label{thm:coconut_final_state}
Let $E_M = h_M - h_M^*$. Then for $M \ge 1$,
$$ \mathbb{E}[\|E_M\|^2] = \frac{1 - L_F^{2M}}{1 - L_F^2} \cdot d\sigma_h^2 > 0. $$
\end{theorem}
\textit{Proof in Appendix~\ref{app:proof_thm_coconut_final_state}.}  The term $\frac{1 - L_F^{2M}}{1 - L_F^2}$ demonstrates that the expected squared error grows with the number of reasoning steps $M$. If the model's transition function is expansive ($L_F > 1$), the error compounds exponentially; even for stable functions ($L_F \le 1$), error still accumulates. This mathematical result directly explains Latent CoT's performance degradation on precision-sensitive tasks like GSM8K. The accumulation of even small perturbations over a long reasoning chain corrupts the final state, undermining the integrity required for symbolic computation and complementing the dichotomy analyzed in Section ref{exploration\_abi}.

\subsection{Unifying the Trade-off via the Symbolic Index}
\label{sec:unifying_perspective}

The analyses in Section \ref{exploration_abi} and Section \ref{exploitation_abli} reveal a dichotomy: CoT excels at execution but lacks exploration, while Latent CoT does the opposite. Here, we unify these behaviors through a single, quantifiable principle: the degree of decisional certainty at each reasoning step. To formalize this, we introduce the \textbf{Symbolic Index} ($\mathcal{I}_{\text{S}}$), a measure of how committed the model is to its top choice.

\begin{definition}[Symbolic Index $\mathcal{I}_{\text{S}}$]
\label{def:s_index}
At a decision point with a vocabulary of valid tokens $\mathcal{V}$, let $p(u \mid h, x)$ be the model's output distribution. The Symbolic Index is the probability of the most likely token:
$$
\mathcal{I}_{\text{S}} = \max_{u \in \mathcal{V}} p(u \mid h, x) \in [1/|\mathcal{V}|, 1].
$$
\end{definition}
A high $\mathcal{I}_{\text{S}}$ ($\approx 1$) signifies high certainty, characteristic of CoT's discrete commitments. A low $\mathcal{I}_{\text{S}}$ signifies distributed consideration over multiple options, characteristic of Latent CoT.

First, we establish the direct link between this certainty and robustness against computational noise. The mechanism for this robustness is the separation between the top two choices, which we term the Logit Decision Margin.

\begin{definition}[Logit Decision Margin $\Delta_l$]
\label{def:logit_margin_revised}
Let $l_{i^*}$ and $l_{j^*}$ be the logits of the most likely and second-most likely tokens, respectively. The margin is their difference: $\Delta_l = l_{i^*} - l_{j^*}$.
\end{definition}

The following theorem proves that high certainty guarantees a large decision margin, making the model resilient to perturbations.

\begin{theorem}[Symbolic Stability Theorem]
\label{thm:symbolic_stability_revised}
The Logit Decision Margin $\Delta_l$ is lower-bounded by the Symbolic Index $\mathcal{I}_{\text{S}}$ as follows:
$$
\Delta_l \geq \log \left( \frac{\mathcal{I}_{\text{S}}}{1 - \mathcal{I}_{\text{S}}} \right).
$$
\end{theorem}
\textit{Proof in Appendix~\ref{app:symbolic_and_margin}.}
This theorem establishes a direct relationship: higher decisional certainty guarantees a larger protective margin. For instance, a high-certainty CoT with $\mathcal{I}_{\text{S}} = 0.99$ has a margin $\Delta_l \geq \log(99) \approx 4.6$, making its decision robust to significant noise. In contrast, a low-certainty Latent CoT with $\mathcal{I}_{\text{S}} = 0.6$ has a margin of only $\Delta_l \geq \log(1.5) \approx 0.4$, rendering it vulnerable to small perturbations. This explains CoT's noise immunity and symbolic integrity (Theorem~\ref{thm:cot_final_state}).

However, this stability comes at a direct and quantifiable cost to exploration. High certainty forces the model's distribution away from the uniform ideal required for unbiased exploration. The next theorem formalizes this trade-off.

\begin{theorem}[Exploration-Execution Trade-off Theorem]
\label{thm:explore_exploit_tradeoff_revised}
For a decision with $B$ valid options, the KL divergence from the ideal uniform exploration prior $q_{\text{PR}}$ is lower-bounded by:
$$
D_{\text{KL}}(q_{\text{PR}} \| p) \geq \log B + \mathcal{I}_{\text{S}}\log(\mathcal{I}_{\text{S}}) + (1-\mathcal{I}_{\text{S}})\log\left(\frac{1-\mathcal{I}_{\text{S}}}{B-1}\right).
$$
This bound is minimized when $\mathcal{I}_{\text{S}} = 1/B$ and grows as $\mathcal{I}_{\text{S}} \to 1$.
\end{theorem}
\textit{Proof in Appendix~\ref{app:tradeoff}.}
This inequality proves that gaining execution stability by increasing $\mathcal{I}_{\text{S}}$ \textit{necessarily} degrades exploration capability by increasing the divergence from a uniform policy. It quantifies the inherent tension between committing to a single path and keeping options open, explaining CoT's exploration deficiency (Theorem~\ref{thm:cot_exploration_deficiency}).

Together, these results establish $\mathcal{I}_{\text{S}}$ as the core regulator, unifying both paradigms:
\begin{enumerate}
    \item \textbf{CoT:} High-$\mathcal{I}_{\text{S}}$ regime ensures large decision margins for robust execution (Theorems \ref{thm:symbolic_stability_revised}, \ref{thm:cot_final_state}) but degrades exploration (Theorem~\ref{thm:explore_exploit_tradeoff_revised}).

    \item \textbf{Latent CoT:} Low-$\mathcal{I}_{\text{S}}$ regime enables broad exploration (Theorem~\ref{thm:latent_cot_exploration_guarantee}) but implies small decision margins, causing vulnerability to sub-decisional noise that compounds over steps (Theorem~\ref{thm:coconut_final_state}), leading to failure on precision-sensitive tasks.
\end{enumerate}

This framework shifts focus from binary architectural choices to \textbf{managing decisional certainty}, suggesting adaptive systems that dynamically regulate $\mathcal{I}_{\text{S}}$ based on task demands.

\section{Why Curriculum Learning Works}
\label{sec:training_dynamics}

Having established that $\mathcal{I}_{\text{S}}$ governs the exploration-execution trade-off in \emph{trained} models (§4.4), we now confront a fundamental question: how can we train a model to operate in the low-$\mathcal{I}_{\text{S}}$ regime required for robust exploration? This reveals a paradox inherent to Latent CoT training. Unlike Explicit CoT, which benefits from token-level supervision at each reasoning step, Latent CoT must learn to generate meaningful latent representations $h_k$ without explicit ground truth.

\subsection{The Training Paradox of Latent CoT}
\label{why_curriculum_works}
    
An untrained model has no incentive to perform multi-step reasoning encoded in $h_k$; instead, it tends to learn "shortcut" representations that merely map surface-level input patterns to outputs. Training requires reliable gradient signals, which necessitate high-certainty (high-$\mathcal{I}_{\text{S}}$) predictions to minimize loss variance. However, the ultimate goal is to learn low-$\mathcal{I}_{\text{S}}$ representations that enable exploration. Direct training thus traps the model in a high-$\mathcal{I}_{\text{S}}$ regime, preventing the emergence of true latent reasoning.

To rigorously analyze this difficulty, we model the problem within an Imitation Learning (IL) framework. Let $P_{\theta^*}(h|x)$ denote the unobserved expert policy that generates optimal, reasoning-encoded latent states. We define a binary valuation function $V(h) \in \{0,1\}$, where $V(h)=1$ implies the latent state leads to the correct answer. The primary objective is to maximize the \textbf{Task Success Rate} $\mathcal{R}(\theta)$, defined as the expected value of this function under the induced policy:
\begin{equation*}
\mathcal{R}(\theta) = \mathbb{E}_{x \sim \mathcal{D}, h \sim P_\theta(\cdot|x)} [V(h)].
\end{equation*}
Without a curriculum, the model must learn from its own generated latent states. These self-generated states come from a biased distribution $P_{\text{biased}}$ of superficial representations, which is fundamentally different from the true expert distribution $P_{\theta^*}$. Since the valuation function $V(h)$ is non-uniform (favoring reasoning states), learning from a mismatched distribution guarantees suboptimal performance.

\begin{theorem}[Provable Failure of Training without Curriculum]
\label{thm:fail_without_curriculum}
Let $D_{nc}$ be a dataset of latent states drawn from a biased distribution $P_{\text{biased}}$. If this distribution is intrinsically suboptimal compared to the expert (i.e., $\mathcal{R}(P_{\text{biased}}) \le \mathcal{R}(\theta^*) - \Delta$ for some gap $\Delta > 0$), then the model $\hat{\theta}_{\text{MLE}}$ trained on $D_{nc}$ is strictly bounded away from optimality:
\[
\mathcal{R}(\hat{\theta}_{\text{MLE}}) \le \mathcal{R}(\theta^*) - C(\Delta),
\]
where $C(\Delta) > 0$ is a constant determined by the bias. This implies that the model's success rate is permanently capped below the expert's performance, regardless of dataset size.
\end{theorem}
\textit{Proof in Appendix~\ref{app:proof_thm_fail_without_curriculum}.}
This theorem demonstrates that training on "shortcut" latent states locks the model into a suboptimal policy. Unlike the variance error which vanishes with data (as seen in Theorem~\ref{thm:success_with_curriculum}), this bias error is irreducible.

\subsection{Curriculum Learning as Controlled $\mathcal{I}_{\text{S}}$ Transition}

From the perspective of the Symbolic Index, this can be understood as a controlled transition. In early stages (pure CoT), high $\mathcal{I}_{\text{S}}$ ensures accurate token-level gradients through error-correcting discretization. In hybrid stages, the prefix operates in latent space (low $\mathcal{I}_{\text{S}}$), but the suffix retains explicit token generation. This anchors the latent representations to valid reasoning traces. Finally, the model achieves the low $\mathcal{I}_{\text{S}}$ required for exploration, stabilized by the progressive internalization.

In our IL framework, this procedure effectively provides supervised samples of latent thoughts grounded in correct reasoning, equivalent to drawing samples directly from $P_{\theta^*}(h|x)$. Crucially, convergence in distribution guarantees convergence in performance. Since the valuation function $V(h)$ is bounded in $[0,1]$, the difference in success rates $|\mathcal{R}(\hat{\theta}) - \mathcal{R}(\theta^*)|$ is upper-bounded by the Total Variation distance between the policies. Consequently, minimizing the distributional divergence via curriculum learning directly optimizes the task success rate.

\begin{theorem}[Provable Success of Training with Curriculum]
\label{thm:success_with_curriculum}
Under standard statistical learning conditions, MLE on a dataset $D_c$ of size $n$ generated via the curriculum yields a model $\hat{\theta}$ whose success rate approaches that of the expert:
\begin{equation*}
\begin{split}
\mathcal{R}(\hat{\theta}) &\ge \mathcal{R}(\theta^*) - \mathcal{O}\left( \sqrt{ \frac{d \log n + \log(1/\delta)}{n} } \right),
\end{split}
\end{equation*}
with high probability $1-\delta$. The performance gap vanishes as the dataset size $n \to \infty$.
\end{theorem}
\textit{Proof in Appendix~\ref{app:proof_thm_success_with_curriculum}.}
This theorem confirms that by using the explicit CoT as a scaffold to generate valid training data, the curriculum ensures that the student model can provably converge to the optimal expert policy. In conclusion, curriculum learning is theoretically necessary to resolve the distributional mismatch (Theorem~\ref{thm:fail_without_curriculum}) and ensure convergence (Theorem~\ref{thm:success_with_curriculum}).

\section{Experiments}

We conduct an empirical evaluation to validate the theoretical framework established in Section \ref{sec:unifying_perspective}. Our experiments verify three hypotheses: (1) Latent CoT and Explicit CoT operate in distinct regimes of decisional certainty (validating Theorems \ref{thm:cot_exploration_deficiency} \& \ref{thm:latent_cot_exploration_guarantee}); (2) this difference in certainty determines robustness to computational noise (validating Theorems \ref{thm:symbolic_stability_revised} \& \ref{thm:explore_exploit_tradeoff_revised}); and (3) Latent CoT exhibits error accumulation consistent with continuous state dynamics (validating Theorem \ref{thm:coconut_final_state}).

\paragraph{Experimental Setup}
We use the GPT-2 (124M) architecture. We compare a standard Explicit CoT model against a Latent CoT model trained via the Coconut curriculum ($M=6$ latent steps). Evaluations are performed on \textbf{GSM8K} \cite{cobbe2021trainingverifierssolvemath} for precise symbolic computation, and \textbf{ProsQA} \cite{hao2024traininglargelanguagemodels} for exploratory reasoning. Implementation details are in Appendix~\ref{app:exp_details}.

\subsection{Analysis of Decisional Certainty ($\mathcal{I}_{\text{S}}$)}
\label{sec:exp_symbolic_index}

To verify the Exploration-Execution Trade-off (Theorem~\ref{thm:explore_exploit_tradeoff_revised}), we measure the \textbf{Symbolic Index} ($\mathcal{I}_{\text{S}}$), defined as the maximum token probability at a given step.

\begin{figure}[h!]
    \centering
    \begin{minipage}{0.48\textwidth}
        \centering
        \includegraphics[width=\linewidth]{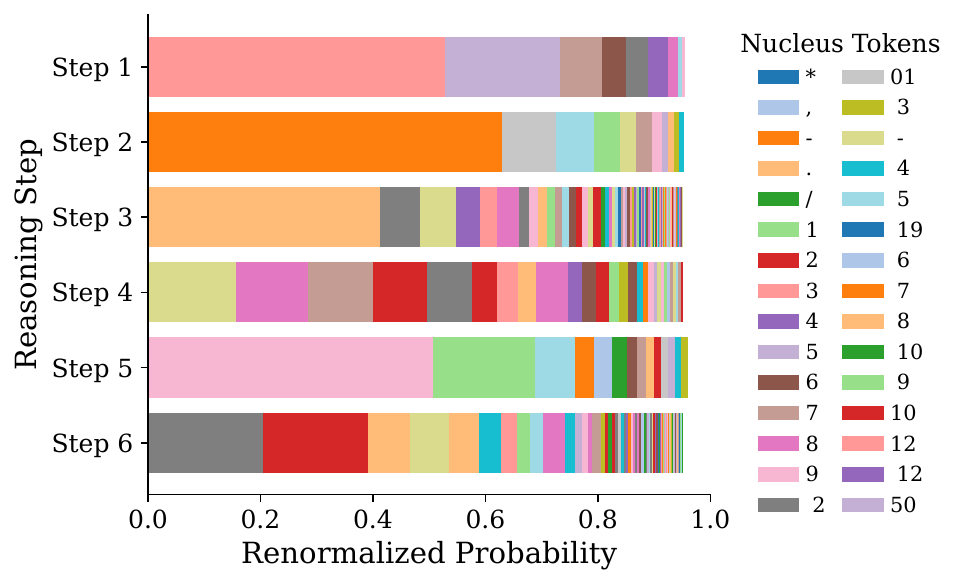}
        \caption{\textbf{Symbolic Index on GSM8K.} Latent CoT (shown) maintains a low Symbolic Index ($\mathcal{I}_{\text{S}} \in [0.2, 0.5]$), indicating a dispersed probability distribution. It lacks the probability concentration ($\mathcal{I}_{\text{S}} \approx 1.0$) observed in Explicit CoT.}
        \label{fig:gsm8k_logits}
    \end{minipage}\hfill
    \begin{minipage}{0.48\textwidth}
        \centering
        \includegraphics[width=\linewidth]{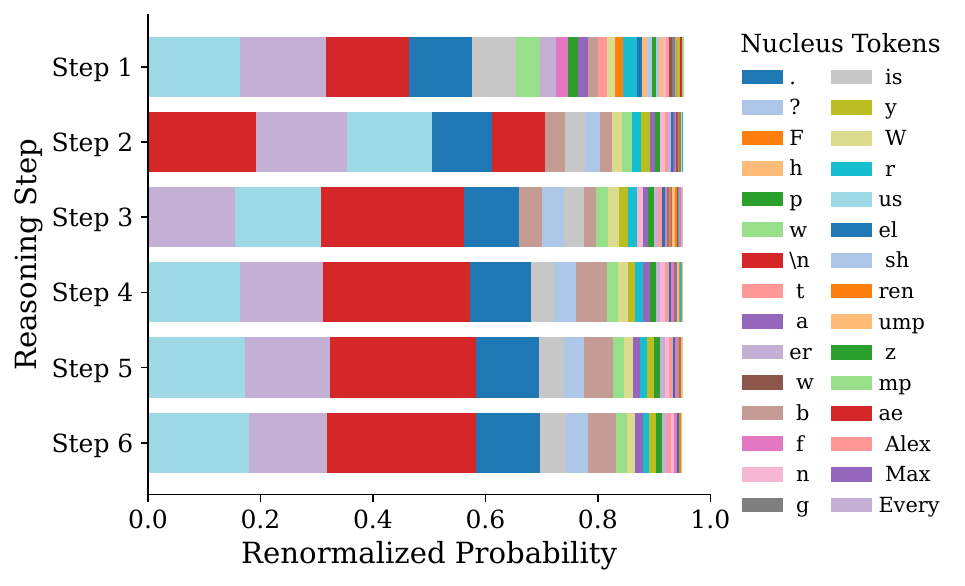}
        \caption{\textbf{Symbolic Index on ProsQA.} Latent CoT exhibits a stable, low $\mathcal{I}_{\text{S}}$ distribution across reasoning steps. This validates Theorem~\ref{thm:latent_cot_exploration_guarantee}, showing that the model distributes probability mass across multiple latent paths rather than converging to a single token.}
        \label{fig:prosqa_logits}
    \end{minipage}
\end{figure}

Figures~\ref{fig:gsm8k_logits} and \ref{fig:prosqa_logits} show the distribution of $\mathcal{I}_{\text{S}}$ values:
\begin{itemize}
    \item \textbf{Exploration via Low $\mathcal{I}_{\text{S}}$ (ProsQA):} On ProsQA (Figure~\ref{fig:prosqa_logits}), the Latent CoT model maintains $\mathcal{I}_{\text{S}} < 0.6$. This empirically confirms Theorem~\ref{thm:latent_cot_exploration_guarantee}. The persistently low $\mathcal{I}_{\text{S}}$ indicates that the model retains multiple potential reasoning trajectories in the latent space, preventing premature convergence to a suboptimal solution. This mechanism supports the high accuracy (97.0\%) observed in exploratory tasks.
    
    \item \textbf{Absence of Discretization (GSM8K):} On GSM8K (Figure~\ref{fig:gsm8k_logits}), the model operates without high-probability concentration ($\mathcal{I}_{\text{S}} \ll 1.0$). According to Theorem~\ref{thm:symbolic_stability_revised}, a low $\mathcal{I}_{\text{S}}$ corresponds to a minimal Logit Decision Margin $\Delta_l$. Unlike Explicit CoT, which utilizes the `argmax` operator to reset state variance at each token, Latent CoT propagates the variance continuously. This lack of intermediate error correction leads to the lower computational accuracy (34.1\%).
\end{itemize}

\subsection{Noise Robustness and Stability Analysis}
\label{sec:exp_noise}

To validate error propagation (Theorem~\ref{thm:coconut_final_state}) and stability bounds (Theorem~\ref{thm:symbolic_stability_revised}), we inject Gaussian noise $\mathcal{N}(0, \sigma^2 I_d)$ into the hidden states (Latent CoT) or pre-output states (CoT) and measure Exact Match accuracy degradation.

\begin{figure}[h!]
    \centering
    \begin{minipage}{0.48\textwidth}
        \centering
        \includegraphics[width=\linewidth]{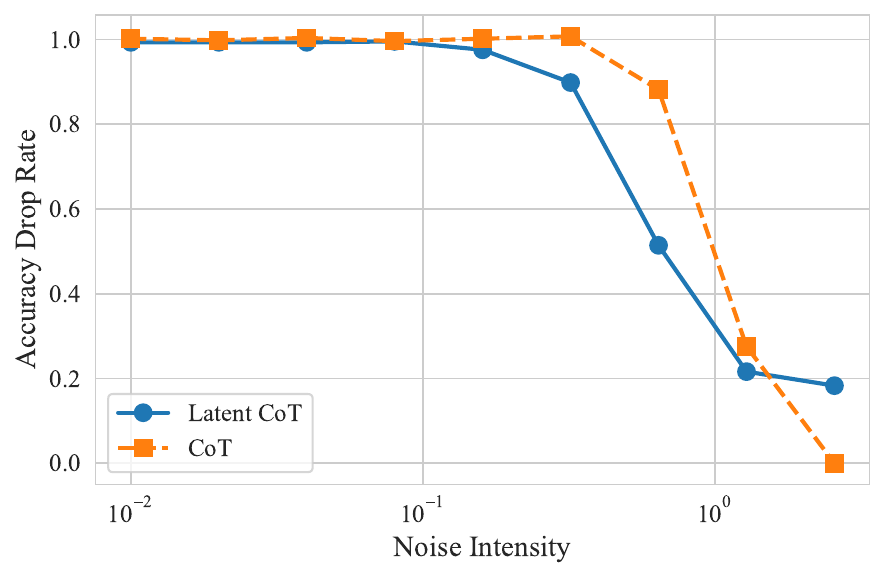} 
\caption{\textbf{Performance Degradation under Noise.} Standard CoT (orange) shows a threshold effect: performance is constant until $\sigma$ exceeds the decision margin. Latent CoT (blue) shows monotonic decay starting from $\sigma \approx 0$. This aligns with the derivation $A(\sigma) = \Phi(\frac{\Delta_l}{\sqrt{C}\sigma})$ in Appendix~\ref{app:accuracy_drop_derivation}.}
        \label{fig:noise_model_comp}
    \end{minipage}\hfill
    \begin{minipage}{0.48\textwidth}
        \centering
        \includegraphics[width=\linewidth]{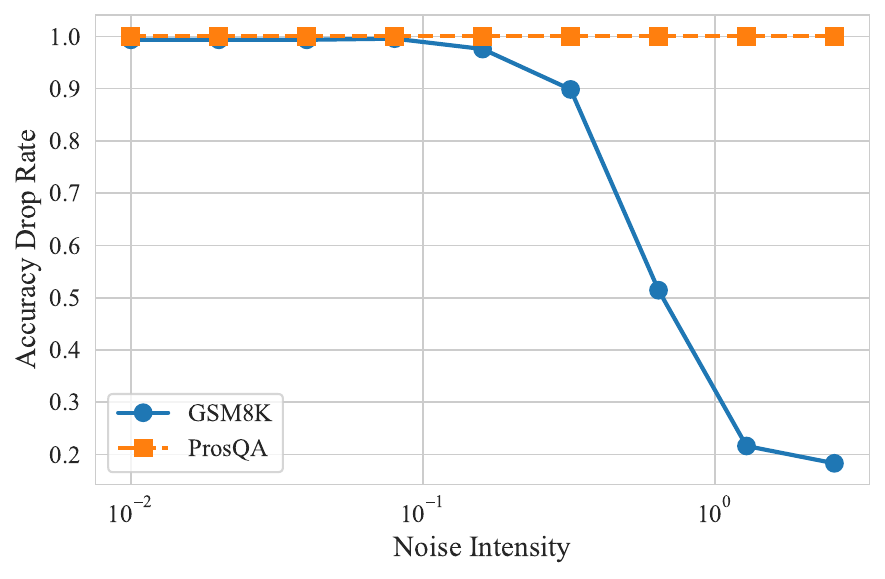} 
        \caption{\textbf{Task-Dependent Sensitivity.} Latent CoT retains higher performance on ProsQA (orange) compared to GSM8K (blue) under identical noise levels. This indicates that exploratory tasks have a larger tolerance for state deviation than precise computational tasks.}
        \label{fig:noise_task_comp}
    \end{minipage}
\end{figure}

The results in Figure~\ref{fig:noise_model_comp} demonstrate two distinct failure modes:
\begin{itemize}
    \item \textbf{Explicit CoT (Step-Function Response):} The Explicit CoT accuracy curve remains flat for low noise levels before exhibiting a sharp decline. This confirms the presence of a non-zero Logit Decision Margin $\Delta_l$. For noise vectors $\|\epsilon\| < \Delta_l$, the discrete projection (token selection) effectively filters the perturbation (Theorem~\ref{thm:cot_final_state}).
    \item \textbf{Latent CoT (Continuous Decay):} Latent CoT exhibits immediate, monotonic degradation. This confirms that $\Delta_l \approx 0$ in the continuous latent representation. Without a discrete margin, perturbations directly alter the state trajectory. The rate of decay on GSM8K supports Theorem~\ref{thm:coconut_final_state}, indicating that the transition function $f_\theta$ accumulates error over the reasoning steps $M=6$.
\end{itemize}

Figure~\ref{fig:noise_task_comp} shows that robustness depends on the task domain. While GSM8K performance degrades rapidly, ProsQA remains relatively stable under noise. This suggests that the solution space for exploratory tasks is continuous (small deviations in $h_M$ still map to valid semantic outputs), whereas computational tasks require $h_M$ to remain within a narrow region to produce the correct numerical token.

\subsection{Ablation: Curriculum Learning}
We analyze the training stability of Latent CoT without curriculum (Table~\ref{tab:main_results}: 52.4\% accuracy on ProntoQA). Monitoring the training process reveals that without the curriculum, the model converges to a high-$\mathcal{I}_{\text{S}}$ regime in the early epochs. This creates a distributional mismatch: the model learns a deterministic mapping based on input-output statistics rather than the expert latent policy $P_{\theta^*}$ (Theorem~\ref{thm:fail_without_curriculum}). The curriculum is therefore necessary to enforce the distribution alignment required for the emergence of the low-$\mathcal{I}_{\text{S}}$ reasoning capability.

\section{Conclusion}

In this work, we introduced the Symbolic Index as a unifying theoretical framework to explain the fundamental trade-off between exploration and execution in language model reasoning. Our analysis identifies decisional certainty as the governing mechanism: Explicit CoT's high certainty secures computational fidelity through error-correcting discretization but limits exploration via premature commitment, whereas Latent CoT's low certainty facilitates broad search at the cost of symbolic integrity due to compounding errors. We further proved that curriculum learning is a theoretical necessity for Latent CoT, ensuring convergence by bridging the distributional mismatch that otherwise traps models in suboptimal policies.

While the Symbolic Index currently serves as a diagnostic metric, it reveals a critical limitation in current paradigm designs: the rigidity of decisional certainty. Our framework suggests that the advancement of reasoning architectures requires moving beyond the binary choice between discrete and continuous modalities. Instead, future research should prioritize adaptive systems capable of dynamically regulating their decisional certainty—modulating between high-certainty commitment for precise execution and low-certainty diversification for planning. By establishing decisional certainty as a central design principle, this work lays the foundation for developing agents that can fluently transition between rigorous calculation and flexible exploration based on task demands.

\section*{Impact Statement}

This paper presents work whose goal is to advance the field of Machine
Learning. There are many potential societal consequences of our work, none
which we feel must be specifically highlighted here.

\bibliography{example_paper}
\bibliographystyle{icml2026}
%%%%%%%%%%%%%%%%%%%%%%%%%%%%%%%%%%%%%%%%%%%%%%%%%%%%%%%%%%%%%%%%%%%%%%%%%%%%%%%
%%%%%%%%%%%%%%%%%%%%%%%%%%%%%%%%%%%%%%%%%%%%%%%%%%%%%%%%%%%%%%%%%%%%%%%%%%%%%%%
% APPENDIX
%%%%%%%%%%%%%%%%%%%%%%%%%%%%%%%%%%%%%%%%%%%%%%%%%%%%%%%%%%%%%%%%%%%%%%%%%%%%%%%
%%%%%%%%%%%%%%%%%%%%%%%%%%%%%%%%%%%%%%%%%%%%%%%%%%%%%%%%%%%%%%%%%%%%%%%%%%%%%%%
\newpage
\appendix
\onecolumn
\section{Appendix}
\subsection{Experimental Details}
\label{app:exp_details}

\noindent\textbf{Models.} Both CoT and Latent CoT (Coconut) models are based on the smallest GPT-2 variant with 124M parameters. For Latent CoT, we use $M=6$ latent reasoning steps across all experiments. The latent thought vector $h_k$ has the same dimension as GPT-2’s hidden size ($d=768$).

\noindent\textbf{Training.} We use the Adam optimizer with learning rate $1\mathrm{e}{-4}$ and weight decay $0.01$. Training is performed on 4 GPUs with mixed-precision (bfloat16 enabled for ProsQA, disabled for GSM8K due to stability). For GSM8K, we train for 25 epochs with batch size 32 and gradient accumulation steps 1; for ProsQA, we train for 30 epochs with batch size 16 and gradient accumulation steps 2. The Coconut curriculum progresses stage-by-stage: for GSM8K, we use 3 epochs per stage up to latent stage 3 (i.e., compressing the first 3 CoT steps into latent thoughts); for ProsQA, we use 4 epochs per stage up to latent stage 6 (full internalization). The base model checkpoint is initialized from standard GPT-2.

\noindent\textbf{Datasets.} We evaluate on two reasoning benchmarks:
\begin{enumerate}
    \item \textbf{GSM8K}: We use the standard split with a held-out test set of 330 examples. Training and validation use the official training set (7,473 examples) and a 10\% validation subset, respectively.
    \item \textbf{ProsQA}: Following \citet{hao2024traininglargelanguagemodels}, we use a test set of 300 examples. The training and validation sets contain 2,000 and 200 examples, respectively.
\end{enumerate}
All reported accuracies are computed on these test sets.

\noindent\textbf{Symbolic Index Visualization Details.} The visualizations in Figures~\ref{fig:gsm8k_logits} and \ref{fig:prosqa_logits} are based on representative examples from the respective test sets.
\begin{itemize}
    \item \textbf{Figure~\ref{fig:gsm8k_logits} (GSM8K):} The plot was generated from the following math reasoning problem: "Janet's ducks lay 16 eggs per day. She eats three for breakfast every morning and bakes muffins for her friends every day with four. She sells the remainder at the farmers' market daily for \$2 per fresh duck egg. How much in dollars does she make every day at the farmers' market?"
    \item \textbf{Figure~\ref{fig:prosqa_logits} (ProsQA):} The plot was generated from the following logical reasoning problem: "Every shumpus is a yumpus. Every worpus is a yimpus. Every shumpus is a gwompus. Every tumpus is a boompus. Every worpus is a shumpus. Every storpus is a terpus. Max is a yimpus. Every shumpus is a rompus. Every wumpus is a jelpus. Every boompus is a terpus. Fae is a tumpus. Every tumpus is a worpus. Every rompus is a gorpus. Every timpus is a impus. Every jompus is a gerpus. Every boompus is a rompus. Fae is a boompus. Every boompus is a kerpus. Every zumpus is a bompus. Max is a rempus. Every rompus is a kerpus. Max is a impus. Every rempus is a impus. Every wumpus is a yumpus. Every grimpus is a terpus. Every tumpus is a jompus. Every yumpus is a felpus. Every jelpus is a felpus. Every shumpus is a felpus. Every rempus is a timpus. Every storpus is a jompus. Every rompus is a storpus. Every tumpus is a wumpus. Every wumpus is a jompus. Every boompus is a worpus. Fae is a storpus. Every worpus is a jelpus. Every grimpus is a felpus. Every worpus is a yumpus. Every rempus is a zumpus. Every kerpus is a grimpus. Is Fae a gwompus or bompus?"
\end{itemize}

\noindent\textbf{Noise Injection Protocol.} During inference, we inject Gaussian noise $\mathcal{N}(0, \sigma^2 I_d)$ as follows:
\begin{enumerate}
    \item For \textbf{CoT}: noise is added to the last hidden state before the LM head, i.e., $h \leftarrow h + \sigma \cdot \epsilon$, where $\epsilon \sim \mathcal{N}(0, I_d)$.
    \item For \textbf{Latent CoT}: noise is added at every reasoning step, i.e., $h_k \leftarrow f_\theta(h_{k-1}) + \sigma \cdot \epsilon_k$, with independent $\epsilon_k \sim \mathcal{N}(0, I_d)$.
\end{enumerate}
We report accuracy as the exact-match rate on the test set. The ''accuracy drop ratio'' in Figure~\ref{fig:noise_model_comp} is defined as $\text{Acc}(\sigma) / \text{Acc}(0)$.

\noindent\textbf{Reproducibility.} All experiments use seed 0. Due to the deterministic nature of our evaluation protocol and the use of nucleus sampling with fixed $p=0.95$ for Symbolic Index estimation, results are fully reproducible given the same model checkpoint.

\subsection{Proof of Theorem~\ref{thm:coconut_cib_duality}}
\label{app:proof_thm_coconut_cib_duality}

\begin{lemma}[Loss and Mutual Information]\label{lemma:loss_mi}
Assume the decoder model family $\{p_\theta(S^{(k+1...M)}\mid h_k,X)\}$ is well-specified, i.e., there exists a parameter $\theta^*$ such that $p_{\theta^*}(S^{(k+1...M)}\mid h_k,X)=p(S^{(k+1...M)} \mid h_k,X)$ holds almost everywhere. Further assume $H(S^{(k+1...M)} \mid X)<\infty$. Then,
\[
\min_\theta\mathcal{L}_{\text{Coconut}}^{(k)}(\theta)\iff\max_{p(h_k \mid X)}I(h_k;S^{(k+1...M)} \mid X).
\]
\end{lemma}

\begin{proof}
Expand the loss function:
\begin{align*}
\mathcal{L}_{\text{Coconut}}^{(k)}(\theta) &= \mathbb{E}_{p(X,S,h_k)} [-\log p_\theta(S^{(k+1...M)} \mid h_k, X)] \\
&= \mathbb{E}_{p(X,h_k)} \left[ \mathbb{E}_{p(S^{(k+1...M)} \mid X,h_k)} [-\log p_\theta(S^{(k+1...M)} \mid h_k, X)] \right] \\
&= \mathbb{E}_{p(X,h_k)} \left[ \mathbb{E}_{p(S^{(k+1...M)} \mid X,h_k)} \left[-\log p(S^{(k+1...M)} \mid h_k,X) + \log \frac{p(S^{(k+1...M)} \mid h_k,X)}{p_\theta(S^{(k+1...M)} \mid h_k,X)}\right] \right] \\
&= \mathbb{E}_{p(X,h_k)} \left[ H(S^{(k+1...M)} \mid h_k,X) + D_{\text{KL}}(p(\cdot \mid h_k,X) \| p_\theta(\cdot \mid h_k,X)) \right].
\end{align*}
By the well-specified assumption, there exists $w^*$ such that $D_{\text{KL}}(p(\cdot \mid h_k,X) \| p_\theta(\cdot \mid h_k,X)) \to 0$. Therefore,
\[
\min_\theta \mathcal{L}_{\text{Coconut}}^{(k)}(\theta) = \min_{p(h_k \mid X)} H(S^{(k+1...M)} \mid h_k, X).
\]
By the definition of conditional mutual information, $I(A;B \mid C) = H(A \mid C) - H(A \mid B,C)$, we have
\[
H(S^{(k+1...M)} \mid h_k, X) = H(S^{(k+1...M)} \mid X) - I(h_k; S^{(k+1...M)} \mid X).
\]
Since $H(S^{(k+1...M)} \mid X)$ is constant, it follows that
\[
\min_{p(h_k \mid X)} H(S^{(k+1...M)} \mid h_k, X) \iff \max_{p(h_k \mid X)} I(h_k; S^{(k+1...M)} \mid X).
\]
\end{proof}
We now begin the proof of Theorem~\ref{thm:coconut_cib_duality}.

\begin{proof}
By Lemma~\ref{lemma:loss_mi}, minimizing the Coconut loss is equivalent to $\min_{p(h_k\mid X)} H(S^{(k+1...M)} \mid h_k, X)$. Any physically realizable encoder is constrained by finite model capacity, formally expressed as $I(h_k;S^{(1...k)} \mid X) \le R$. We construct the Lagrangian:
\[
\mathcal{J}(p, \lambda) = H(S^{(k+1...M)} \mid h_k, X) + \lambda \left( I(h_k;S^{(1...k)} \mid X) - R \right), \quad \lambda \ge 0.
\]
Using the identity $H(A \mid B, C) = H(A \mid C) - I(B; A \mid C)$, we rewrite:
\[
\mathcal{J} = \left[ H(S^{(k+1...M)} \mid X) - I(h_k;S^{(k+1...M)} \mid X) \right] + \lambda I(h_k;S^{(1...k)} \mid X) - \lambda R.
\]
Ignoring constants $H(S^{(k+1...M)} \mid X)$ and $\lambda R$ during minimization over $p(h_k\mid X)$, we obtain:
\[
\min_{p(h_k\mid X)} \left\{ \lambda I(h_k;S^{(1...k)} \mid X) - I(h_k;S^{(k+1...M)} \mid X) \right\}.
\]
Defining $\beta = 1/\lambda > 0$ and dividing through by $\lambda$, this becomes the standard CIB objective:
\[
\min_{p(h_k\mid X)} \left\{ I(h_k;S^{(1...k)} \mid X) - \beta I(h_k;S^{(k+1...M)} \mid X) \right\}.
\]

Denote $I_{\text{past}} = I(h_k; S^{(1...k)} \mid X)$ and $I_{\text{future}} = I(h_k; S^{(k+1...M)} \mid X)$. Any Pareto-optimal encoder lies on the efficiency frontier of the information plane, where the trade-off between retained and predicted information satisfies the first-order condition:
\[
\frac{d I_{\text{future}}}{d I_{\text{past}}} \bigg|_{h_k^*} = \frac{1}{\beta(k)}.
\]

We model the frontier as:
\[
I_{\text{future}}(I_{\text{past}}; k) = I_{\max}(k) \left(1 - e^{-\alpha(k) \cdot I_{\text{past}}}\right),
\]
with parameters motivated by:
\item $I_{\max}(k) \approx C_2 (M - k)$, proportional to the length of the remaining sequence.
\item $\alpha(k) \approx 1 / H(S^{(1...k)} \mid X) \approx 1 / (C_1 k + C_0)$, inversely proportional to the information content of past tokens.
Differentiating gives:
\[
\frac{d I_{\text{future}}}{d I_{\text{past}}} = I_{\max}(k) \cdot \alpha(k) \cdot e^{-\alpha(k) \cdot I_{\text{past}}}.
\]
At optimality, the encoder saturates its bottleneck: $I_{\text{past}} \approx H(S^{(1...k)} \mid X) \approx 1 / \alpha(k)$. Substituting:
\[
\frac{d I_{\text{future}}}{d I_{\text{past}}} \bigg|_{h_k^*} \approx I_{\max}(k) \cdot \alpha(k) \cdot e^{-1} \approx \frac{C_2 (M - k)}{e (C_1 k + C_0)}.
\]
Applying $\beta(k) = \left( d I_{\text{future}} / d I_{\text{past}} \right)^{-1}$:
\[
\beta(k) = e \cdot \frac{C_1 k + C_0}{C_2 (M - k)} = e \frac{C_1}{C_2} \cdot \frac{k}{M - k} + O\left(\frac{1}{M - k}\right).
\]
Thus, as $k \to M$, the dominant asymptotic behavior is $\beta(k) \sim \frac{k}{M - k}$, revealing that the optimal trade-off scales with the ratio of processed to remaining sequence length.
\end{proof}

\subsection{Proof of Theorem~\ref{thm:cot_exploration_deficiency}}
\label{app:proof_thm_cot_exploration_deficiency}

\begin{lemma}[Upper Bound on the Entropy $H(p_{\text{CoT}})$ of the CoT Output Distribution]\label{lemma:cot_entropy_bound}
Let $\bar{p} = \mathbb{E}[p_{\text{CoT}}]$ be the expected next-step distribution over $B$ options, drawn from a Dirichlet prior with parameters $\boldsymbol{\alpha}$ and concentration $\kappa = \sum_{i=1}^B \alpha_i$. Under high concentration — i.e., when $\alpha_j \gg \alpha_{i \ne j}$ for some $j$ — the entropy $H(\bar{p})$ vanishes as $\kappa$ grows. Specifically, if $\alpha_j = \kappa - (B-1)c$ and $\alpha_{i \ne j} = c$ for small $c > 0$, then
$$
H(\bar{p}) = O\left(\frac{\log \kappa}{\kappa}\right),
$$
\end{lemma}

\begin{proof}
By properties of the Dirichlet distribution, the expected probabilities are given by $\bar{p}_i = \alpha_i / \kappa$. We consider a typical unimodal scenario in which the vast majority of the weight concentrates on a single option $j$. We set $\alpha_j = \kappa - (B-1)c$ and $\alpha_{i \ne j} = c$ for all $i \ne j$, where $c$ is a small positive constant.

Under this setting, the expected probabilities are:
\begin{align*}
    \bar{p}_j &= \frac{\kappa - (B-1)c}{\kappa} = 1 - \frac{(B-1)c}{\kappa} \\
    \bar{p}_{i \ne j} &= \frac{c}{\kappa}
\end{align*}
We compute the entropy of this distribution:
\begin{align*}
    H(\bar{p}) &= -\bar{p}_j \log \bar{p}_j - \sum_{i \ne j} \bar{p}_{i} \log \bar{p}_{i} \\
    &= -\left(1 - \frac{(B-1)c}{\kappa}\right) \log\left(1 - \frac{(B-1)c}{\kappa}\right) - (B-1)\frac{c}{\kappa} \log\left(\frac{c}{\kappa}\right)
\end{align*}
We analyze the first term using the Taylor expansion $\log(1-x) \approx -x$. As $\kappa \to \infty$, we have $\frac{(B-1)c}{\kappa} \to 0$, and thus:
\[
-\left(1 - \frac{(B-1)c}{\kappa}\right) \log\left(1 - \frac{(B-1)c}{\kappa}\right) \approx -\left(1 - \frac{(B-1)c}{\kappa}\right) \left(-\frac{(B-1)c}{\kappa}\right) = O\left(\frac{1}{\kappa}\right).
\]
Next, we analyze the second term:
\begin{align*}
    -(B-1)\frac{c}{\kappa} \log\left(\frac{c}{\kappa}\right) &= -(B-1)\frac{c}{\kappa}(\log c - \log \kappa) \\
    &= (B-1)c \frac{\log \kappa}{\kappa} - (B-1)c\frac{\log c}{\kappa} = O\left(\frac{\log \kappa}{\kappa}\right).
\end{align*}
Combining both terms, as $\kappa$ becomes large, the entropy is dominated by the $O\left(\frac{\log \kappa}{\kappa}\right)$ term. And we have $\lim_{\kappa\to\infty} \frac{\log \kappa}{\kappa} = 0$.
\end{proof}

We now begin the proof of Theorem~\ref{thm:cot_exploration_deficiency}.

\begin{proof}
We analyze the almost sure asymptotic behavior of $D_{\text{KL}}(q_{\text{PR}} \| p_{\text{CoT}})$ under the concentrated Dirichlet prior where $\alpha_j = \kappa - (B-1)c$ and $\alpha_{i \ne j} = c$.

By the law of large numbers for Dirichlet distributions, as $\kappa \to \infty$, the random vector $p_{\text{CoT}}$ converges almost surely to its mean:
\[
p_{\text{CoT}} \xrightarrow{a.s.} \bar{p} = \left(1 - \frac{(B-1)c}{\kappa}, \frac{c}{\kappa}, \ldots, \frac{c}{\kappa}\right).
\]

Since the KL divergence is a continuous function of the probability vector (away from the boundary), we have:
\[
D_{\text{KL}}(q_{\text{PR}} \| p_{\text{CoT}}) \xrightarrow{a.s.} D_{\text{KL}}(q_{\text{PR}} \| \bar{p}) \quad \text{as } \kappa \to \infty.
\]

Now compute $D_{\text{KL}}(q_{\text{PR}} \| \bar{p})$:
\begin{align*}
D_{\text{KL}}(q_{\text{PR}} \| \bar{p}) 
&= \sum_{i=1}^B \frac{1}{B} \log \frac{1/B}{\bar{p}_i} \\
&= -\log B - \frac{1}{B} \sum_{i=1}^B \log \bar{p}_i \\
&= -\log B - \frac{1}{B} \left[ \log\left(1 - \frac{(B-1)c}{\kappa}\right) + (B-1)\log\left(\frac{c}{\kappa}\right) \right].
\end{align*}

Using Taylor expansion $\log(1 - x) = -x + O(x^2)$:
\begin{align*}
D_{\text{KL}}(q_{\text{PR}} \| \bar{p}) 
&= -\log B - \frac{1}{B} \left[ -\frac{(B-1)c}{\kappa} + O\left(\frac{1}{\kappa^2}\right) + (B-1)\log c - (B-1)\log \kappa \right] \\
&= \frac{B-1}{B}\log\kappa - \log B - \frac{(B-1)\log c}{B} + O\left(\frac{1}{\kappa}\right).
\end{align*}

Therefore, almost surely as $\kappa \to \infty$:
\[
D_{\text{KL}}(q_{\text{PR}} \| p_{\text{CoT}}) = \frac{B-1}{B}\log\kappa - \log B - \frac{(B-1)\log c}{B} + \mathcal{O}\left(\frac{1}{\kappa}\right).
\]

Since $H(p_{\text{CoT}}) \to 0$ as established in Lemma~\ref{lemma:cot_entropy_bound}, and $D_{\text{KL}}(q_{\text{PR}} \| p_{\text{CoT}})$ grows logarithmically in $\kappa$, this proves the exploration deficiency of CoT.
\end{proof}

\subsection{Proof of Theorem~\ref{thm:latent_cot_exploration_guarantee}}
\label{app:proof_thm_latent_cot_exploration_guarantee}

\begin{assumption}[Convergence and Compactness of Latent States]
\label{assump:state_convergence}
We assume that during inference, the sequence of latent states $\{h_k\}$ converges to a fixed point or enters a compact attracting set as the reasoning step $k$ increases. This is not just a theoretical convenience; it is a behavior we consistently observe empirically. As shown in Figure~\ref{fig:pca_convergence}, a PCA visualization of the latent thought embeddings reveals a clear convergent trajectory, where successive states $\{L1, \dots, L6\}$ move progressively closer within a bounded region. The high explained variance (0.99) of the first two principal components confirms that this 2D projection faithfully represents the dynamics in the original high-dimensional space.
\end{assumption}

\begin{figure}[h!]
    \centering
    \begin{subfigure}{0.49\textwidth}
        \centering
        \includegraphics[width=\linewidth]{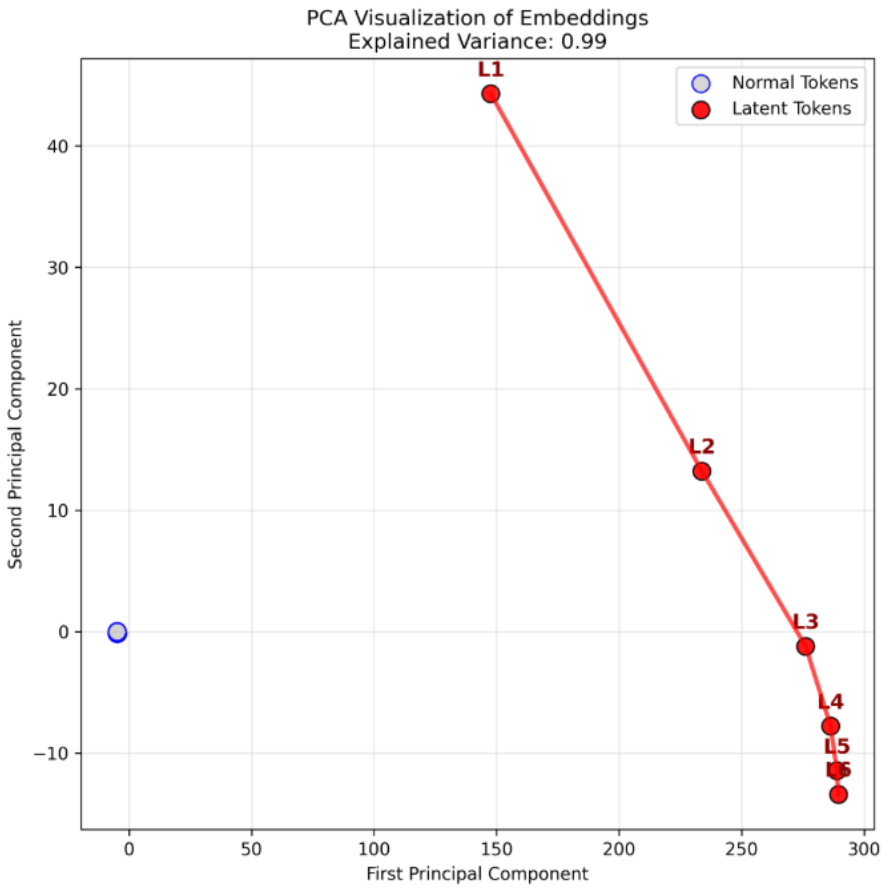}
        \caption{PCA visualization on GSM8K task.}
        \label{fig:pca_gsm8k}
    \end{subfigure}
    \hfill
    \begin{subfigure}{0.49\textwidth}
        \centering

        \includegraphics[width=\linewidth]{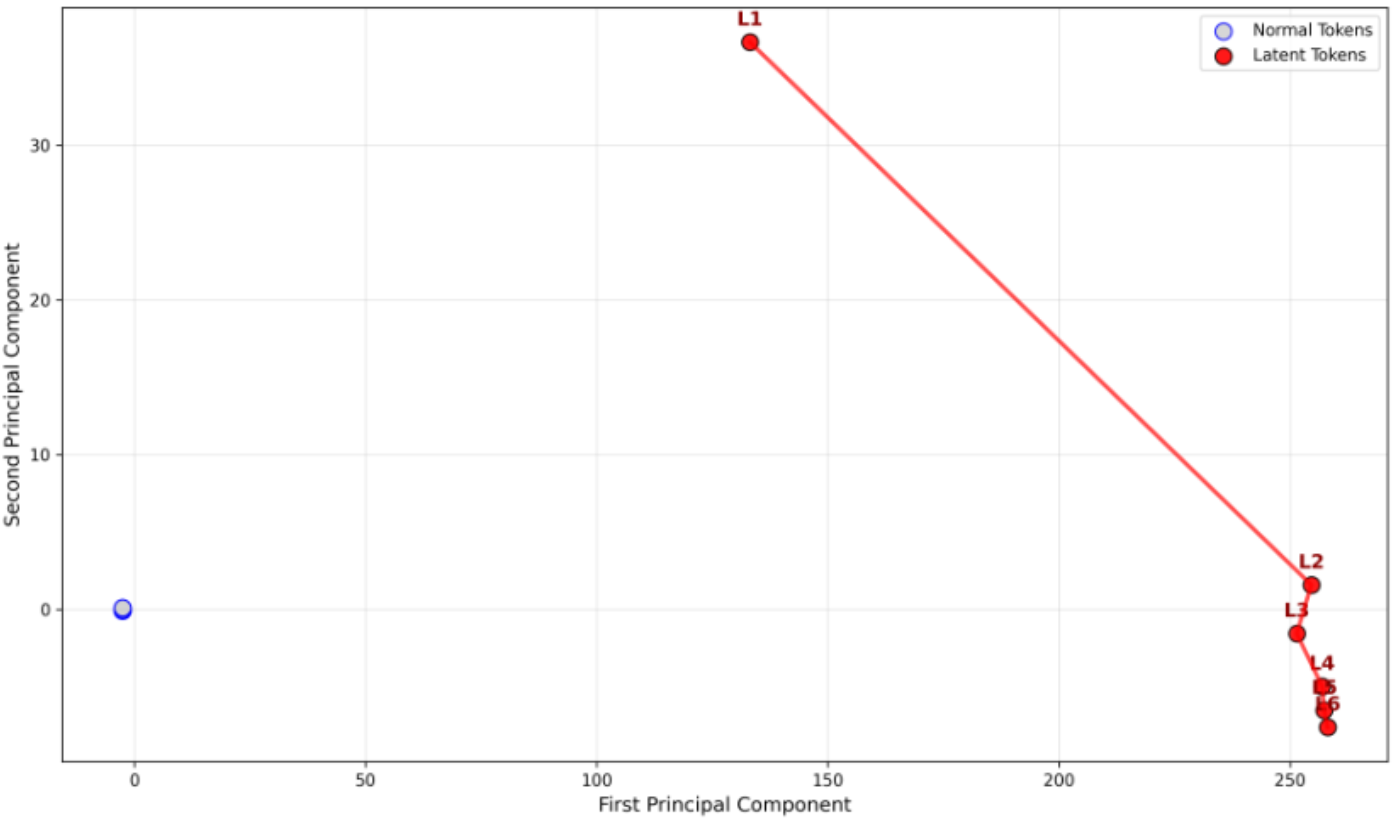} 
        \caption{PCA visualization on ProsQA task.}
        \label{fig:pca_prosqa}
    \end{subfigure}

    \caption{PCA visualization of latent state embeddings for reasoning trajectories. In both (a) and (b), the latent thoughts (L1-L6) demonstrate a clear convergence, supporting our assumption of a compact, attracting set for latent states. This validates that the convergence dynamic is a general property of the latent reasoning process, not specific to a single task.}
    \label{fig:pca_convergence}
\end{figure}

\begin{lemma}[Non-Degenerate Output Distribution of Coconut]\label{lem:coconut_non_degenerate}
For a model trained via the Coconut objective, its optimization can be modeled as solving a Conditional Information Bottleneck (CIB) problem with a finite trade-off parameter $\beta(k) > 0$ (Theorem~\ref{thm:coconut_cib_duality}). Under this framework, the model's output probability distribution is non-degenerate: there exists a constant $\delta > 0$ such that for any valid next token $u$ and any reachable latent state $h$, the probability is strictly bounded away from 1:
$$
\sup_{h \in \mathcal{H}} \left( \max_{u \in N_{\text{valid}}(v)} p(u \mid h)\right) \le 1 - \delta,
$$
where $\mathcal{H}$ is the compact latent state space guaranteed by Assumption~\ref{assump:state_convergence}.
\end{lemma}

\begin{proof}
Assume, for contradiction, that $\sup_{h \in \mathcal{H}} \max_{u} p(u \mid h) = 1$. Since $\mathcal{H}$ is compact and the decoder $p(u \mid h)$ is continuous, there must exist a latent state $h^* \in \mathcal{H}$ and a token $u^*$ such that $p(u^* \mid h^*) = 1$.

At $h^*$, future uncertainty is fully eliminated: $H(S^{(k+1...M)} \mid h^*, X) = 0$, implying that the predictive information $I(h; S^{(k+1...M)} \mid X)$ achieves its theoretical maximum.

By Theorem~\ref{thm:coconut_cib_duality}, the Coconut solution corresponds to an optimal solution of a CIB problem. On the information plane $(I_{\text{past}}, I_{\text{future}})$, all optimal solutions must lie on the efficiency frontier, and any optimal point $h_{\text{opt}}$ must satisfy:
\[
\frac{d I_{\text{future}}}{d I_{\text{past}}} \bigg|_{h_{\text{opt}}} = \frac{1}{\beta(k)}.
\]

At our hypothetical point $h^*$, $I_{\text{future}}$ reaches its maximum value $I_{\max}(k)$. Under the efficiency frontier model adopted in Theorem~\ref{thm:coconut_cib_duality}:
\[
I_{\text{future}}(I_{\text{past}}; k) = I_{\max}(k) \left(1 - e^{-\alpha(k) \cdot I_{\text{past}}}\right),
\]
which is strictly increasing and strictly concave in $I_{\text{past}}$, with derivative:
\[
\frac{d I_{\text{future}}}{d I_{\text{past}}} = I_{\max}(k) \cdot \alpha(k) \cdot e^{-\alpha(k) \cdot I_{\text{past}}}.
\]
As $I_{\text{future}} \to I_{\max}(k)$, we must have $I_{\text{past}} \to \infty$, and the derivative $\to 0^+$. If $I_{\text{future}}$ were to reach $I_{\max}(k)$ at a finite $I_{\text{past}}$ (as at $h^*$), the derivative must be zero to match the asymptotic behavior. Thus, at $h^*$:
\[
\frac{d I_{\text{future}}}{d I_{\text{past}}} \bigg|_{h^*} = 0.
\]
If $h^*$ were a CIB optimum, it must satisfy:
\[
\frac{d I_{\text{future}}}{d I_{\text{past}}} \bigg|_{h^*} = \frac{1}{\beta(k)},
\]
implying $1/\beta(k) = 0$, i.e., $\beta(k) \to \infty$. However, by Theorem~\ref{thm:coconut_cib_duality}, at any non-terminal stage of curriculum learning ($k < M$), $\beta(k)$ is a finite positive constant (with asymptotic behavior $\beta(k) \sim k/(M-k) < \infty$), so $1/\beta(k) > 0$. This contradicts the requirement that the slope be zero.

Therefore, the assumption $p(u^* \mid h^*) = 1$ leads to a contradiction. Hence, $\sup \max p(u \mid h)$ must be strictly less than 1. The lemma is proved.

\noindent\textbf{Intuition:} In essence, Coconut inherently avoids deterministic collapse by design: its information bottleneck objective enforces a trade-off that preserves output diversity, even at advanced reasoning stages.
\end{proof}
We now begin the proof of Theorem~\ref{thm:latent_cot_exploration_guarantee}
\begin{proof}
By Lemma~\ref{lem:coconut_non_degenerate}, there exists $\delta > 0$ such that for any valid next token $u$ and reachable state $h$, the probability satisfies $\max_u p(u \mid h) \leq 1 - \delta$. Our goal is to derive an upper bound for the KL divergence $D_{\text{KL}}(q_{\text{PR}} \| p_{\text{Coconut}})$.

By definition, 
\[
D_{\text{KL}}(q_{\text{PR}} \| p) = -\log B - \frac{1}{B} \sum_{u} \log p(u \mid v).
\]
To bound this quantity from above, we seek the probability distribution that maximizes $D_{\text{KL}}$ under the constraints $\sum p_i = 1$ and $\max p_i \le 1 - \delta$, which is equivalent to minimizing $\sum \log p(u \mid v)$.

The minimum of $\sum \log p_i$ is achieved by the most concentrated distribution satisfying the constraint:
\[
p(u^* \mid v) = 1 - \delta, \quad \text{and for all } u \ne u^*, \quad p(u \mid v) = \frac{\delta}{B-1}.
\]
Substituting this worst-case distribution into the KL divergence expression yields an upper bound:
\begin{align*}
    D_{\text{KL}}(q_{\text{PR}} \| p_{\text{Coconut}}) &\leq -\log B - \frac{1}{B} \left[ \log(1 - \delta) + (B-1) \log \left( \frac{\delta}{B-1} \right) \right] \\
    &= -\frac{B-1}{B} \log\delta - \left[ \log B - \frac{B-1}{B} \log(B-1) + \frac{\log(1 - \delta)}{B} \right].
\end{align*}
Denote the bracketed term as $f(B) = \log B - \frac{B-1}{B} \log(B-1) + \frac{\log(1 - \delta)}{B}$. Since $f(B)$ is continuous and bounded on $[2, \infty)$ (with $\lim_{B\to\infty} f(B) = 0$), its infimum $c = \inf_{B \ge 2} f(B)$ is a finite constant. Therefore, $-f(B)$ is upper-bounded by $-c$.

Meanwhile, since $B \ge 2$, we have $\frac{B-1}{B} \ge \frac{1}{2}$. Also, since $\delta \in (0, 1)$, it follows that $\log\delta < 0$, and thus $-\frac{B-1}{B}\log\delta \le -\frac{1}{2}\log\delta$.

Combining these bounds, the KL divergence admits a finite upper bound independent of the branching factor $B$:
\[
D_{\text{KL}}(q_{\text{PR}} \| p_{\text{Coconut}}) \leq -\frac{B-1}{B} \log\delta - c \leq -\frac{1}{2}\log\delta - c.
\]
\end{proof}

\subsection{Proof of Theorem~\ref{thm:cot_final_state}}
\label{app:proof_thm_cot_final_state}

\begin{proof}
We prove that under sub-decisional perturbations, the Explicit CoT trajectory $\hat{S}$ is identical to the noise-free trajectory $S^*$ by induction on the reasoning step $k$.

Let $f_\theta(\cdot)$ denote the deterministic forward pass of the model. At step $k$, the model takes the input $x$ and the discrete history $S^{(1 \dots k-1)}$ to produce a logit vector $l_k^* = f_\theta(x, S^{(1 \dots k-1)})$.
The noise-free token selection is given by $s_k^* = \operatorname{argmax}(l_k^*)$.
The perturbed process generates logits $\hat{l}_k = l_k^* + \epsilon_k$, where $\epsilon_k$ satisfies the sub-decisional condition: $\operatorname{argmax}(l_k^* + \epsilon_k) = \operatorname{argmax}(l_k^*)$.

At the first step, both the clean and perturbed processes receive the exact same input $x$ (with an empty history). The clean logits are $l_1^* = f_\theta(x, \emptyset)$. The perturbed logits are $\hat{l}_1 = l_1^* + \epsilon_1$.
By the definition of sub-decisional perturbation:
$$
\hat{s}_1 = \operatorname{argmax}(\hat{l}_1) = \operatorname{argmax}(l_1^* + \epsilon_1) = \operatorname{argmax}(l_1^*) = s_1^*.
$$
Thus, the first token is identical, and the history passed to step 2 is error-free.

Assume that for all steps $t < k$, the generated tokens are identical, i.e., $\hat{s}_t = s_t^*$. This implies the discrete history is identical: $\hat{S}^{(1 \dots k-1)} = S^{*(1 \dots k-1)}$.
At step $k$, since the model $f_\theta$ is deterministic and conditioned strictly on the discrete history:
$$
\text{Logits}_{\text{pre-noise}} = f_\theta(x, \hat{S}^{(1 \dots k-1)}) = f_\theta(x, S^{*(1 \dots k-1)}) = l_k^*.
$$
The noise process adds the perturbation: $\hat{l}_k = l_k^* + \epsilon_k$.
Applying the sub-decisional constraint again, the discrete projection filters out the noise:
$$
\hat{s}_k = \operatorname{argmax}(l_k^* + \epsilon_k) = s_k^*.
$$
Consequently, the noise term $\epsilon_k$ is discarded by the argmax operator and does not propagate to the state of the next step.

By induction, $\hat{s}_k = s_k^*$ for all $k=1, \dots, M$. Therefore, the probability of trajectory divergence is zero:
$$ \mathbb{P}[\hat{S} \neq S^*] = 0. $$
This mathematically confirms the \textit{discretization-reset} mechanism described in Section~\ref{exploitation_abli}: the argmax operator effectively acts as a non-linear filter that resets the internal state to a clean integer grid at every step, preventing the accumulation of sub-decisional noise.
\end{proof}

\subsection{Proof of Theorem~\ref{thm:coconut_final_state}}
\label{app:proof_thm_coconut_final_state}

\begin{proof}
From the recurrence
\[
E_k = f_\theta(h_{k-1}) - f_\theta(h_{k-1}^*) + \epsilon_h^{(k)},
\]
taking the squared norm and expectation, and using the independence and zero-mean property of the noise (which causes cross terms to vanish), we obtain:
\[
\mathbb{E}[\|E_k\|^2] \le L_F^2 \, \mathbb{E}[\|E_{k-1}\|^2] + d\sigma_h^2.
\]
This is a non-homogeneous linear recurrence. Solving it with initial condition $\mathbb{E}[\|E_0\|^2] = 0$ yields:
\[
\mathbb{E}[\|E_M\|^2] = 
\begin{cases}
\displaystyle \frac{1 - L_F^{2M}}{1 - L_F^2} \cdot d\sigma_h^2, & L_F \neq 1, \\[0.8em]
M \cdot d\sigma_h^2, & L_F = 1.
\end{cases}
\]
In both cases, $\mathbb{E}[\|E_M\|^2] > 0$ for all $M \ge 1$, which completes the proof.
\end{proof}

\subsection{Derivation of the Normalized Accuracy Function}
\label{app:accuracy_drop_derivation}

In this section, we provide a rigorous derivation of the \textbf{Normalized Accuracy} curve $A(\sigma)$ to explain the characteristic inverted S-shape (sigmoid-like) degradation observed in Figure~\ref{fig:noise_model_comp}. We model the causal chain from injected noise in the latent space to the probabilistic failure of token ranking in the output space.

We begin by modeling the accumulation of noise in the hidden states. Let $h_M^*$ denote the clean latent state at the final reasoning step $M$. Consistent with Theorem~\ref{thm:coconut_final_state} (Compounding Error), the injection of Gaussian noise $\epsilon \sim \mathcal{N}(0, \sigma^2 I_d)$ at each step accumulates over the reasoning trajectory. Assuming a linearized approximation of the transition dynamics, the final perturbed state $\tilde{h}_M$ can be expressed as:
\begin{equation*}
    \tilde{h}_M = h_M^* + \epsilon_{\text{acc}}, \quad \text{where} \quad \epsilon_{\text{acc}} \sim \mathcal{N}(0, \sigma_{\text{eff}}^2 I_d).
\end{equation*}
Here, $\epsilon_{\text{acc}}$ represents the effective accumulated noise vector in $\mathbb{R}^d$. The variance $\sigma_{\text{eff}}^2 \approx \alpha \cdot \sigma^2$ scales linearly with the injection variance $\sigma^2$, where the scaling factor $\alpha$ depends on the path length $M$ and the Lipschitz constant of the network.

Since discrete token decisions are not made in the latent space but in the logit space, we must trace how this perturbation propagates through the unembedding layer. Let $W_U \in \mathbb{R}^{|\mathcal{V}| \times d}$ be the weight matrix of the language model head. The perturbed logits $\tilde{l}$ are obtained by projecting the final latent state:
\begin{equation*}
    \tilde{l} = W_U \tilde{h}_M = W_U (h_M^* + \epsilon_{\text{acc}}) = l^* + \eta,
\end{equation*}
where $l^*$ denotes the clean logits, and $\eta = W_U \epsilon_{\text{acc}}$ represents the noise projected into the logit space. Since $\epsilon_{\text{acc}}$ is Gaussian and $W_U$ is a linear transformation, $\eta$ follows a multivariate Gaussian distribution over $\mathbb{R}^{|\mathcal{V}|}$.

A prediction is successful if the noise is insufficient to alter the ranking of the top tokens. Let $i^* = \operatorname{argmax}(l^*)$ be the index of the correct token, and let $j^*$ be the index of the strongest competitor. The model retains the correct prediction if $\tilde{l}_{i^*} > \tilde{l}_{j^*}$. Substituting $\tilde{l} = l^* + \eta$, this condition becomes:
\begin{equation*}
    l^*_{i^*} + \eta_{i^*} > l^*_{j^*} + \eta_{j^*} \implies \eta_{j^*} - \eta_{i^*} < l^*_{i^*} - l^*_{j^*}.
\end{equation*}
Crucially, the Right-Hand Side is precisely the Logit Decision Margin $\Delta_l$ defined in Definition~\ref{def:logit_margin_revised}. Therefore, the success condition simplifies to:
\begin{equation*}
    \xi < \Delta_l, \quad \text{where} \quad \xi = \eta_{j^*} - \eta_{i^*}.
\end{equation*}

To quantify the probability of this event, we analyze the distribution of the random variable $\xi$. We express $\xi$ as a linear combination of the accumulated noise vector: $\xi = (w_{j^*} - w_{i^*})^\top \epsilon_{\text{acc}}$. Because $\epsilon_{\text{acc}}$ is isotropic Gaussian, $\xi \sim \mathcal{N}(0, \sigma_1^2)$. The scalar variance $\sigma_1^2$ is given by:
\begin{equation*}
    \sigma_1^2 = \|w_{j^*} - w_{i^*}\|_2^2 \cdot \sigma_{\text{eff}}^2.
\end{equation*}
Letting $C = \alpha \cdot \|w_{j^*} - w_{i^*}\|_2^2$, we derive the relationship $\sigma_1^2 = C \cdot \sigma^2$. Thus, $\sigma_1 = \sqrt{C} \cdot \sigma$.

The Normalized Accuracy $A(\sigma)$ is defined as the probability of success $\mathbb{P}(\xi < \Delta_l)$. We evaluate this using the CDF of the standard normal distribution, $\Phi(z)$. By standardizing via $Z = \xi / \sigma_1 \sim \mathcal{N}(0, 1)$, we obtain:
\begin{align}
    A(\sigma) &= \mathbb{P}(\xi < \Delta_l) \nonumber \\
    &= \mathbb{P}\left( \frac{\xi}{\sigma_1} < \frac{\Delta_l}{\sigma_1} \right) \nonumber \\
    &= \Phi\left( \frac{\Delta_l}{\sqrt{C} \cdot \sigma} \right).
\end{align}

This derivation provides a rigorous physical interpretation for the experimental results. In the limit as $\sigma \to 0$, the argument $\frac{\Delta_l}{\sqrt{C}\sigma} \to \infty$, and $A(\sigma) \to 1$. This explains the ``sub-decisional plateau,'' where accuracy remains stable because the noise is strictly contained within the safety margin $\Delta_l$. As $\sigma$ increases, the argument decreases, leading to the characteristic sigmoid-like decay observed in the accuracy curve. This confirms that robustness is analytically determined by the ratio between the Symbolic Index-induced margin $\Delta_l$ and the accumulated noise variance.

\subsection{Proof of Theorem~\ref{thm:symbolic_stability_revised}}
\label{app:symbolic_and_margin}
\begin{proof}
The Logit Decision Margin $\Delta_l$ is defined as the difference between the logits of the most and second-most likely tokens, $\Delta_l = l_{i^*} - l_{j^*}$. We aim to establish a lower bound for this quantity based on the Symbolic Index, $\mathcal{I}_{\text{S}}$.

First, we relate the logits to probabilities using the softmax function, $p_i = e^{l_i} / \sum_k e^{l_k}$. The ratio of the probabilities of the top two tokens is:
\[
\frac{p_{i^*}}{p_{j^*}} = \frac{e^{l_{i^*}}}{e^{l_{j^*}}} = e^{l_{i^*} - l_{j^*}} = e^{\Delta_l}.
\]
By taking the natural logarithm of both sides, we can express the logit margin in terms of probabilities:
\[
\Delta_l = \log(p_{i^*}) - \log(p_{j^*}).
\]
By the definition of the Symbolic Index, the probability of the most likely token is $p_{i^*} = \mathcal{I}_{\text{S}}$.

Next, we establish an upper bound for the probability of the second-most likely token, $p_{j^*}$. The sum of probabilities over all possible tokens is 1. Therefore, the sum of probabilities of all tokens except the most likely one is $1 - p_{i^*} = 1 - \mathcal{I}_{\text{S}}$. Since $p_{j^*}$ is the largest among these remaining probabilities, it cannot be greater than their sum. Thus, we have the inequality:
\[
p_{j^*} \le 1 - \mathcal{I}_{\text{S}}.
\]
Because the logarithm function is monotonically increasing, this implies $\log(p_{j^*}) \le \log(1 - \mathcal{I}_{\text{S}})$. Consequently, its negative, $-\log(p_{j^*})$, satisfies the reverse inequality: $-\log(p_{j^*}) \ge -\log(1 - \mathcal{I}_{\text{S}})$.

Finally, we substitute our findings back into the equation for $\Delta_l$:
\[
\Delta_l = \log(\mathcal{I}_{\text{S}}) - \log(p_{j^*}) \ge \log(\mathcal{I}_{\text{S}}) - \log(1 - \mathcal{I}_{\text{S}}).
\]
Using the properties of logarithms, we arrive at the desired lower bound:
\[
\Delta_l \ge \log\left(\frac{\mathcal{I}_{\text{S}}}{1 - \mathcal{I}_{\text{S}}}\right).
\]
\end{proof}

\subsection{Proof of Theorem~\ref{thm:explore_exploit_tradeoff_revised}}
\label{app:tradeoff}
\begin{proof}
The KL divergence between the ideal uniform prior $q_{\text{PR}}$ (where $q_i = 1/B$ for all $i=1,\dots,B$) and the model's output distribution $p$ is given by:
\[
D_{\text{KL}}(q_{\text{PR}} \| p) = \sum_{i=1}^B q_i \log\left(\frac{q_i}{p_i}\right) = \sum_{i=1}^B \frac{1}{B} \log\left(\frac{1/B}{p_i}\right) = -\log B - \frac{1}{B}\sum_{i=1}^B \log p_i.
\]
To find a lower bound for $D_{\text{KL}}(q_{\text{PR}} \| p)$, we must find an upper bound for the term $\sum_{i=1}^B \log p_i$, subject to the constraints imposed by the model's distribution. The constraints are:
\begin{enumerate}
    \item $\sum_{i=1}^B p_i = 1$
    \item $\max_i p_i = \mathcal{I}_{\text{S}}$
\end{enumerate}
We want to solve the following optimization problem:
\[
\max_{p_1, \dots, p_B} \sum_{i=1}^B \log p_i \quad \text{s.t.} \quad \sum p_i = 1 \text{ and } \max p_i = \mathcal{I}_{\text{S}}.
\]
Let $p_{i^*} = \mathcal{I}_{\text{S}}$ be the maximal probability. The sum of the remaining $B-1$ probabilities is $\sum_{k \neq i^*} p_k = 1 - \mathcal{I}_{\text{S}}$. The function $f(p_1, \dots, p_{B-1}) = \sum_{k \neq i^*} \log p_k$ is concave. By Jensen's inequality, this sum is maximized when the remaining probabilities are distributed as uniformly as possible, i.e., when $p_k = \frac{1 - \mathcal{I}_{\text{S}}}{B-1}$ for all $k \neq i^*$.

This distribution, which makes the probabilities as "flat" as possible given the peak at $\mathcal{I}_{\text{S}}$, provides the maximum possible value for $\sum \log p_i$. Substituting this extremal distribution, we find the upper bound:
\[
\max \sum_{i=1}^B \log p_i = \log(\mathcal{I}_{\text{S}}) + (B-1) \log\left(\frac{1 - \mathcal{I}_{\text{S}}}{B-1}\right).
\]
Now, we substitute this upper bound back into the expression for the KL divergence. Since we are subtracting this term, its upper bound yields a lower bound for the KL divergence:
\[
D_{\text{KL}}(q_{\text{PR}} \| p) \ge -\log B - \frac{1}{B} \left[ \log(\mathcal{I}_{\text{S}}) + (B-1) \log\left(\frac{1 - \mathcal{I}_{\text{S}}}{B-1}\right) \right].
\]
This completes the proof.
\end{proof}

\subsection{Proof of Theorem~\ref{thm:fail_without_curriculum}}
\label{app:proof_thm_fail_without_curriculum}

\begin{proof}
We prove this by constructing a counterexample. We define a reasoning task where the biased training distribution has a strictly lower success rate than the expert. We then show that the Maximum Likelihood Estimator (MLE) converges to a policy that mimics this suboptimal behavior, leading to a permanent performance gap.

\paragraph{1. Instance Construction}
Consider a $d=3$ dimensional latent space with three discrete states:
\begin{itemize}
    \item $h_{\text{expert}}$: The correct reasoning path. Feature $\phi(h_{\text{expert}}) = [1, 0, 0]^\top$.
    \item $h_{\text{shortcut}}$: A superficial shortcut. Feature $\phi(h_{\text{shortcut}}) = [0, 1, 1]^\top$.
    \item $h_{\text{bad}}$: Irrelevant noise. Feature $\phi(h_{\text{bad}}) = [0, 1, 0]^\top$.
\end{itemize}
The valuation function $V(h)$ (correctness of the answer derived from $h$) is defined as:
\[
V(h) = \begin{cases} 1 & \text{if } h = h_{\text{expert}} \\ 0 & \text{otherwise} \end{cases}
\]
The \textbf{Task Success Rate} for a policy parameterized by $\theta$ is thus $\mathcal{R}(\theta) = P_\theta(h_{\text{expert}})$.

\paragraph{2. The Expert vs. The Biased Distribution}
Let the expert parameter be $\theta^* = [10, 0, 0]^\top$. The unnormalized log-probability (score) for the expert state is 10, while others are 0. The expert's success rate is:
\[
\mathcal{R}(\theta^*) = \frac{e^{10}}{e^{10} + e^0 + e^0} \approx 1.0.
\]
Now, consider the biased data distribution $P_{\text{biased}}$. Assume the dataset consists entirely of shortcut reasoning, i.e., $P_{\text{biased}}(h_{\text{shortcut}}) = 1$.
The success rate of this data distribution is:
\[
\mathcal{R}(P_{\text{biased}}) = \mathbb{E}_{h \sim P_{\text{biased}}}[V(h)] = 1 \cdot V(h_{\text{shortcut}}) = 0.
\]
This satisfies the theorem's condition $\mathcal{R}(P_{\text{biased}}) \le \mathcal{R}(\theta^*) - \Delta$ with gap $\Delta \approx 1$.

\paragraph{3. MLE Convergence and the Performance Gap}
We train a model $\hat{\theta}_{\text{MLE}}$ on this biased dataset. MLE maximizes the log-likelihood of the observed data ($h_{\text{shortcut}}$). To maximize the score $\langle \theta, \phi(h_{\text{shortcut}}) \rangle = \theta_2 + \theta_3$, while adhering to normalization constraints, the optimizer will drive $\theta_2 + \theta_3 \to \infty$.

A characteristic solution for large $N$ is $\hat{\theta}_{\text{MLE}} \approx [0, K, K]^\top$ for a large $K$. Under this learned policy, the probability mass shifts entirely to the shortcut:
\[
P_{\hat{\theta}_{\text{MLE}}}(h_{\text{shortcut}}) \propto e^{2K}, \quad P_{\hat{\theta}_{\text{MLE}}}(h_{\text{expert}}) \propto e^0 = 1.
\]
The success rate of the learned model is:
\[
\mathcal{R}(\hat{\theta}_{\text{MLE}}) = P_{\hat{\theta}_{\text{MLE}}}(h_{\text{expert}}) = \frac{1}{1 + e^{2K} + e^K} \xrightarrow{K \to \infty} 0.
\]
\paragraph{4. Conclusion}
Comparing the final success rates:
\[
\mathcal{R}(\hat{\theta}_{\text{MLE}}) \approx 0 \quad \text{vs.} \quad \mathcal{R}(\theta^*) \approx 1.
\]
Thus, we have derived the bound:
\[
\mathcal{R}(\hat{\theta}_{\text{MLE}}) \le \mathcal{R}(\theta^*) - C,
\]
where $C \approx 1$ is a strictly positive constant derived from the initial gap $\Delta$. This proves that without a curriculum to correct the distributional bias, the model's performance is permanently bounded away from optimality.
\end{proof}

\subsection{Proof of Theorem~\ref{thm:success_with_curriculum}}
\label{app:proof_thm_success_with_curriculum}

\paragraph{Justification for the Imitation Learning Framework}
Our theoretical analysis, particularly Theorem~\ref{thm:success_with_curriculum}, models the training process within the standard framework of Imitation Learning (IL). This framework assumes access to an i.i.d. dataset of input-latent state pairs $D_c = \{(x_i, h_i)\}_{i=1}^n$ sampled from an expert latent policy $P_{\theta^*}(h|x)$. However, the Coconut training paradigm is supervised on predicting future tokens from expert trajectories $D_S = \{(x_i, S_i^*)\}_{i=1}^n$. This section provides a formal justification for the equivalence between training a Coconut model and performing Maximum Likelihood Estimation (MLE) in this IL setting.

The core argument is as follows:

\begin{enumerate}
    \item \textbf{Definition of the Implicit Expert Latent Policy:}
    By Theorem~\ref{thm:coconut_cib_duality} (Coconut-CIB Duality), the objective of Coconut training is to learn an encoder $E_\theta$ such that the latent state $h_k = E_\theta(S^{(1...k)})$ becomes a \textbf{Minimal Sufficient Statistic} of the past $S^{(1...k)}$ for predicting the future $S^{(k+1...M)}$.
    Let $E_{\theta^*}$ be the optimal encoder that solves this CIB problem. We can define an \textbf{implicit expert latent policy} $P_{\theta^*}(h|x)$ as the distribution induced by applying the optimal encoder to the distribution of expert trajectory prefixes:
    \[        h \sim P_{\theta^*}(h|x) \quad \text{where} \quad h = E_{\theta^*}(S^{*}_{\text{past}}) \text{ and } S^{*}_{\text{past}} \sim P_{\text{expert}}(S_{\text{past}}|x).
\]
    The support of this policy, $P_{\theta^*}(h|x)$, constitutes the space of ideal latent states.

    \item \textbf{Equivalence of the Optimization Process:}
    The Coconut objective is to minimize the expected negative log-likelihood:
    \[        \min_{\theta} \mathcal{L}(\theta) = \mathbb{E}_{(x, S^*) \sim P_{\text{expert}}} [-\log p(S^{*}_{\text{future}} | E_{\theta}(S^{*}_{\text{past}}), x)].
    \]
    By the Data Processing Inequality, for any encoder $E_\theta$, we have:
    \[        I(S^{*}_{\text{past}}; S^{*}_{\text{future}} | x) \ge I(E_{\theta}(S^{*}_{\text{past}}); S^{*}_{\text{future}} | x).
\]
    The optimal encoder $E_{\theta^*}$ achieves equality, as it preserves all predictive information. This implies that the conditional distribution factorizes as $p(S^{*}_{\text{future}} | S^{*}_{\text{past}}, x) = p(S^{*}_{\text{future}} | E_{\theta^*}(S^{*}_{\text{past}}), x) = p(S^{*}_{\text{future}} | h^*, x)$.
    Therefore, the loss $\mathcal{L}(\theta)$ effectively measures the divergence in predictive power between the latent distribution induced by $E_\theta$ and the ideal latent distribution $P_{\theta^*}(h|x)$. The loss is minimized if and only if the induced latent distribution $p_\theta(h|x)$ converges to $P_{\theta^*}(h|x)$.
\end{enumerate}

\textbf{Conclusion:}
The Coconut training process, while formally supervised on future token generation, has an intrinsic CIB objective that \textbf{compels} the distribution of latent states generated by the encoder, $p_\theta(h|x)$, to fit the implicit expert latent policy $P_{\theta^*}(h|x)$. Consequently, minimizing the Coconut loss on the expert trajectory dataset $D_S$ is mathematically equivalent to performing MLE on an implicit dataset $D_c$ constructed from $h_i^* = E_{\theta^*}(S_{i, \text{past}}^*)$. This provides a rigorous foundation for applying our IL-based theoretical guarantees to the Coconut model.

\begin{lemma}[Self-Normalized Bound for Vector Martingales (\cite{NIPS2011_e1d5be1c})]
\label{lem:self_normalized_bound}
Let $\{\mathcal{F}_i\}_{i=0}^n$ be a filtration. Let $\{g_i\}_{i=1}^n$ be a sequence of random vectors in $\mathbb{R}^d$ such that $g_i$ is $\mathcal{F}_i$-measurable and $\mathbb{E}[g_i \mid \mathcal{F}_{i-1}] = 0$. Assume there exists a constant $\sigma > 0$ such that for any unit vector $v \in \mathbb{R}^d$, the real-valued random variable $\langle v, g_i \rangle$ is conditionally $\sigma$-sub-Gaussian, i.e.,
\[
\forall \lambda \in \mathbb{R}, \quad \mathbb{E}[\exp(\lambda \langle v, g_i \rangle) \mid \mathcal{F}_{i-1}] \le \exp\left(\frac{\lambda^2 \sigma^2}{2}\right).
\]
Let $V$ be a $d \times d$ positive definite matrix. Define
\[
V_n = V + \sum_{i=1}^n g_i g_i^\top \quad \text{and} \quad G_n = \sum_{i=1}^n g_i.
\]
Then, for any $\delta \in (0, 1)$, with probability at least $1 - \delta$, the following inequality holds:
\[
\|G_n\|_{V_n^{-1}}^2 \le 2\sigma^2 \log\left( \frac{\det(V_n)^{1/2} \det(V)^{-1/2}}{\delta} \right).
\]
\end{lemma}
\begin{lemma}[Construction of a Confidence Set for Parameter Estimation]
\label{lemma:confidence_set}
Let $\hat{\theta}_{\text{MLE}}$ be the parameter estimate obtained by maximizing the log-likelihood function on a curriculum dataset $D_c = \{(x_i, h_i)\}_{i=1}^n$, where the data is sampled i.i.d. from the expert distribution $P_{\theta^*}$:
\[
\hat{\theta}_{\text{MLE}} = \arg\max_{\theta \in \mathbb{R}^d} \sum_{i=1}^n \log P_{\theta}(h_i \mid x_i)
\]
Assume the following regularity conditions for the log-linear model hold:

\begin{enumerate}
    \item \textbf{(Bounded Features)}: For all $(x, h)$, $\|\phi(x, h)\|_2 \le L$.
    \item \textbf{(Strong Convexity)}: The expected negative log-likelihood function $L(\theta) = \mathbb{E}_{x,h \sim P_{\theta^*}}[-\log P_{\theta}(h|x)]$ is $\lambda_{\min}$-strongly convex with respect to $\theta$ over the parameter space. Its Hessian (the Fisher information matrix) $\nabla^2 L(\theta) = I(\theta)$ satisfies $I(\theta) \succeq \lambda_{\min} I$.
    \item \textbf{(Sub-Gaussian Score)}: The score function at the true parameter $\theta^*$, $\nabla_\theta \log P_{\theta^*}(h \mid x)$, is a zero-mean, $\sigma^2$-sub-Gaussian random vector.
    \item \textbf{(Bounded Parameters)}: $\|\theta^*\|_2 \le B$, and the optimization is performed within the compact set $\mathcal{B} = \{\theta \in \mathbb{R}^d : \|\theta\|_2 \le B\}$.
\end{enumerate}

Then for any $\delta \in (0,1)$, there exists a constant $C > 0$ (depending only on $L, \sigma, \lambda_{\min}, B$) such that with probability at least $1 - \delta$, the true expert parameter $\theta^*$ satisfies:
\[
\left\lVert \hat{\theta}_{\text{MLE}} - \theta^* \right\rVert_{\hat{H}_n}^2 \le \beta_n^2(\delta)
\]
where:
\begin{itemize}
    \item $\hat{H}_n = \frac{1}{n} \sum_{i=1}^n \nabla^2 (-\log P_{\hat{\theta}_{\text{MLE}}}(h_i \mid x_i))$ is the empirical Hessian evaluated at $\hat{\theta}_{\text{MLE}}$,
    \item $\beta_n^2(\delta) = C \cdot \frac{d \log(n) + \log(1/\delta)}{n}$ is the squared radius of the confidence set.
\end{itemize}
In other words, the confidence set $\Theta(\hat{\theta}_{\text{MLE}}) = \left\{ \theta \in \mathbb{R}^d : \left\lVert \theta - \hat{\theta}_{\text{MLE}} \right\rVert_{\hat{H}_n}^2 \le \beta_n^2(\delta) \right\}$ covers the true parameter $\theta^*$ with high probability.
\end{lemma}

\begin{proof}
This proof combines the first-order optimality condition, concentration inequalities, and self-normalized analysis techniques to construct a non-asymptotic, high-probability bound.

\paragraph{1. Notation and Basic Relations}
Let $L_n(\theta) = \frac{1}{n}\sum_{i=1}^n \ell_i(\theta)$ be the empirical negative log-likelihood, where $\ell_i(\theta) = -\log P_{\theta}(h_i|x_i)$. By definition, $\hat{\theta}_{\text{MLE}}$ is the minimizer of $L_n(\theta)$ within the parameter space $\mathcal{B}$.
Since $\theta^*$ is the minimizer of the expected risk $L(\theta)$ (because the KL divergence $D_{KL}(P_{\theta^*} \| P_{\theta}) = L(\theta) - L(\theta^*)$ is non-negative), we have $\nabla L(\theta^*) = 0$.

\paragraph{2. Using Optimality and Convexity}
By the optimality of $\hat{\theta}_{\text{MLE}}$, we have $L_n(\hat{\theta}_{\text{MLE}}) \le L_n(\theta^*)$.
Since $L_n(\theta)$ is a convex function, we can derive a basic inequality that connects the parameter error to an empirical process:
\begin{align*}
    L(\hat{\theta}_{\text{MLE}}) - L(\theta^*) &\le L(\hat{\theta}_{\text{MLE}}) - L(\theta^*) - (L_n(\hat{\theta}_{\text{MLE}}) - L_n(\theta^*)) \\
    &= (L - L_n)(\hat{\theta}_{\text{MLE}}) - (L - L_n)(\theta^*) \\
    &= \frac{1}{n} \sum_{i=1}^n \left( \mathbb{E}[\ell_i(\hat{\theta}_{\text{MLE}})] - \ell_i(\hat{\theta}_{\text{MLE}}) \right) - \left( \mathbb{E}[\ell_i(\theta^*)] - \ell_i(\theta^*) \right)
\end{align*}
By Assumption 2, $L(\theta)$ is $\lambda_{\min}$-strongly convex, so $L(\hat{\theta}_{\text{MLE}}) - L(\theta^*) \ge \frac{\lambda_{\min}}{2} \|\hat{\theta}_{\text{MLE}} - \theta^*\|_2^2$. Combining this with the above inequality yields:
$$\frac{\lambda_{\min}}{2} \|\hat{\theta}_{\text{MLE}} - \theta^*\|_2^2 \le \sup_{\theta \in \mathcal{B}} |(L - L_n)(\theta) - (L - L_n)(\theta^*)|
$$
Bounding the empirical process on the right-hand side can yield a preliminary $L_2$-norm convergence bound of $\mathcal{O}(\sqrt{d/n})$. However, to obtain a sharper, weighted-norm bound, we turn to self-normalized analysis.

\paragraph{3. Applying Self-Normalized Bounds}
We directly analyze the error relation derived from the first-order condition. Let $\Delta = \hat{\theta}_{\text{MLE}} - \theta^*$. A first-order Taylor expansion of $\nabla L_n(\hat{\theta}_{\text{MLE}}) = 0$ around $\theta^*$ gives:
\[
0 = \nabla L_n(\hat{\theta}_{\text{MLE}}) = \nabla L_n(\theta^*) + \int_0^1 \nabla^2 L_n(\theta^* + t\Delta) dt \cdot \Delta
\]
Rearranging, we get:
\begin{equation}
\label{eq:error_relation}
\bar{H}_n \Delta = - \nabla L_n(\theta^*)
\end{equation}
where $\bar{H}_n = \int_0^1 \nabla^2 L_n(\theta^* + t\Delta) dt$ is the mean Hessian matrix on the line segment between $\theta^*$ and $\hat{\theta}_{\text{MLE}}$.

The core of the analysis lies in bounding the norm of the empirical score $\nabla L_n(\theta^*)$. Let $g_i = \nabla_\theta \ell_i(\theta^*) = \nabla_\theta (-\log P_{\theta^*}(h_i|x_i))$. Since $\mathbb{E}[g_i | \mathcal{F}_{i-1}] = \nabla L(\theta^*) = 0$, the sequence $\{g_i\}_{i=1}^n$ is a vector-valued martingale difference sequence. Under our regularity conditions (specifically, Assumption 3), $\{g_i\}$ is also conditionally sub-Gaussian. We can therefore apply the self-normalized bound from Lemma~\ref{lem:self_normalized_bound}.

Let's map the variables to Lemma~\ref{lem:self_normalized_bound}: we set $V = \lambda I$ for some small regularization constant $\lambda > 0$, which ensures initial positive definiteness. Then we have $G_n = \sum_{i=1}^n g_i$ and $V_n = \lambda I + \sum_{i=1}^n g_i g_i^\top$. Applying the lemma, we get that with probability at least $1-\delta$:
\[
\left\| \sum_{i=1}^n g_i \right\|_{(\lambda I + \sum_{i=1}^n g_i g_i^\top)^{-1}}^2 \le 2\sigma^2 \log\left( \frac{\det(\lambda I + \sum g_i g_i^\top)^{1/2} \det(\lambda I)^{-1/2}}{\delta} \right).
\]
Under the bounded feature assumption, we can bound the determinant term and simplify the expression. Using the fact that $\sum g_i g_i^\top \approx n \cdot \mathbb{E}[gg^\top] = n \cdot I(\theta^*)$ and properties of the determinant, one can derive that:
\begin{equation}
\label{eq:grad_bound}
\| \nabla L_n(\theta^*) \|_{(I(\theta^*))^{-1}}^2 = \frac{1}{n^2} \left\| \sum_{i=1}^n g_i \right\|_{(I(\theta^*))^{-1}}^2 \lesssim \frac{d \log(n) + \log(1/\delta)}{n}
\end{equation}
where $\lesssim$ hides logarithmic factors and constants dependent on model parameters.

\paragraph{4. Connecting the Error and Gradient Bounds}
Returning to Eq. \eqref{eq:error_relation}, we have $\Delta = -\bar{H}_n^{-1} \nabla L_n(\theta^*)$. We want to bound $\|\Delta\|_{\hat{H}_n}^2$.
By matrix concentration inequalities and Assumptions 1 and 4, it can be shown that the empirical Hessian $\nabla^2 L_n(\theta)$ concentrates uniformly around its expectation $I(\theta)$ over the parameter space $\mathcal{B}$. This implies that both $\bar{H}_n$ and $\hat{H}_n$ are close to $I(\theta^*)$ with high probability. Formally, there exists $\epsilon_n \to 0$ such that, with high probability, $\| \bar{H}_n - I(\theta^*) \|_{\text{op}} \le \epsilon_n$ and $\| \hat{H}_n - I(\theta^*) \|_{\text{op}} \le \epsilon_n$.
Therefore, we can approximate:
\begin{align*}
    \|\Delta\|_{\hat{H}_n}^2 &= \Delta^\top \hat{H}_n \Delta \\
    &= (\nabla L_n(\theta^*))^\top \bar{H}_n^{-1} \hat{H}_n \bar{H}_n^{-1} \nabla L_n(\theta^*) \\
    &\approx (\nabla L_n(\theta^*))^\top (I(\theta^*))^{-1} \nabla L_n(\theta^*) \\
    &= \| \nabla L_n(\theta^*) \|_{(I(\theta^*))^{-1}}^2
\end{align*}
By more rigorously handling the approximation $\bar{H}_n^{-1} \hat{H}_n \bar{H}_n^{-1} \approx (I(\theta^*))^{-1}$ and controlling the error terms, combined with Eq. \eqref{eq:grad_bound}, it can ultimately be shown that:
\[
\left\lVert \hat{\theta}_{\text{MLE}} - \theta^* \right\rVert_{\hat{H}_n}^2 \le C \cdot \frac{d \log(n) + \log(1/\delta)}{n}
\]
where the constant $C$ absorbs all terms dependent on $L, \sigma, \lambda_{\min}, B$. This completes the proof of the lemma.
\end{proof}

We can now prove the main theorem of this section.

\begin{proof}
The proof proceeds in three steps: first, we leverage Lemma \ref{lemma:confidence_set} to bound the KL divergence between the model distribution $P_{\hat{\theta}}$ and the expert distribution $P_{\theta^*}$; second, we convert the KL divergence bound into a Total Variation (TV) distance bound using Pinsker's inequality; finally, we use the properties of TV distance to bound the difference in task success rates between the two policies.

\paragraph{1. Bounding the KL Divergence}
We aim to bound the KL divergence $D_{KL}(P_{\theta^*} \| P_{\hat{\theta}})$. For the log-linear model family we consider, the KL divergence has the form:
$$
    D_{KL}(P_{\theta^*}(\cdot|x) \| P_{\hat{\theta}}(\cdot|x)) = \mathbb{E}_{h \sim P_{\theta^*}(\cdot|x)}[\log P_{\theta^*}(h|x) - \log P_{\hat{\theta}}(h|x)]
$$
Performing a second-order Taylor expansion of $\log P_{\hat{\theta}}(h|x)$ around $\theta^*$, we have:
\[
\log P_{\hat{\theta}}(h|x) \approx \log P_{\theta^*}(h|x) + \langle \nabla_\theta \log P_{\theta^*}(h|x), \Delta \rangle + \frac{1}{2} \Delta^\top \nabla^2_\theta \log P_{\theta^*}(h|x) \Delta
\]
where $\Delta = \hat{\theta} - \theta^*$. Substituting this into the KL divergence expression and taking the expectation, we note that $\mathbb{E}_{h \sim P_{\theta^*}}[\nabla_\theta \log P_{\theta^*}] = 0$, which yields the quadratic approximation of KL divergence in terms of the Fisher information matrix:
$$
    \mathbb{E}_{x \sim \rho} [D_{KL}(P_{\theta^*} \| P_{\hat{\theta}})] \approx \frac{1}{2} \Delta^\top I(\theta^*) \Delta = \frac{1}{2} \|\hat{\theta} - \theta^*\|_{I(\theta^*)}^2
$$
More rigorously, this approximation can be shown to hold with controllable higher-order terms. Since Lemma \ref{lemma:confidence_set} shows that $\|\hat{\theta} - \theta^*\|_{\hat{H}_n}$ is small with high probability and $\hat{H}_n$ concentrates around $I(\theta^*)$, we can translate the result of Lemma \ref{lemma:confidence_set} into a bound on the KL divergence. With probability at least $1-\delta$:
\begin{equation}
\label{eq:kl_bound}
    \mathbb{E}_{x \sim \rho} [D_{KL}(P_{\theta^*} \| P_{\hat{\theta}})] \le C' \cdot \|\hat{\theta} - \theta^*\|_{\hat{H}_n}^2 \le C \cdot \frac{d \log n + \log(1/\delta)}{n}
\end{equation}
where $C'$ and $C$ are positive constants.

\paragraph{2. From KL Divergence to Total Variation Distance}
The Total Variation (TV) distance is defined as $TV(P, Q) = \frac{1}{2} \int |p(x) - q(x)| dx$. Pinsker's inequality establishes a relationship between KL divergence and TV distance:
$$
TV(P, Q) \le \sqrt{\frac{1}{2} D_{KL}(P \| Q)}
$$
Applying this inequality to our model and expert distributions, and combining it with Eq. \eqref{eq:kl_bound}, we obtain a high-probability bound on the TV distance:
\begin{align}
    \mathbb{E}_{x \sim \rho} [TV(P_{\theta^*} \| P_{\hat{\theta}})] &\le \mathbb{E}_{x \sim \rho} \left[ \sqrt{\frac{1}{2} D_{KL}(P_{\theta^*} \| P_{\hat{\theta}})} \right] \nonumber \\
    &\le \sqrt{\frac{1}{2} \mathbb{E}_{x \sim \rho} [D_{KL}(P_{\theta^*} \| P_{\hat{\theta}})]} \quad (\text{by Jensen's inequality}) \nonumber \\
    &\le \sqrt{\frac{C}{2} \cdot \frac{d \log n + \log(1/\delta)}{n}} \nonumber \\
    &= \mathcal{O}\left(\sqrt{\frac{d \log n + \log(1/\delta)}{n}}\right)
\label{eq:tv_bound}
\end{align}

\paragraph{3. From Total Variation Distance to Task Success Rate}
A key property of the TV distance is that it bounds the difference in expectation for any bounded function. Our task success function $V(h) \in \{0, 1\}$ is a bounded function (with bound 1). Therefore,
\begin{align}
    |\text{SuccessRate}(\hat{\theta}) - \text{SuccessRate}(\theta^*)| &= \left| \mathbb{E}_{x \sim \rho, h \sim P_{\hat{\theta}}} [V(h)] - \mathbb{E}_{x \sim \rho, h \sim P_{\theta^*}} [V(h)] \right| \nonumber \\
    &= \left| \mathbb{E}_{x \sim \rho} \left[ \mathbb{E}_{h \sim P_{\hat{\theta}}} [V(h)] - \mathbb{E}_{h \sim P_{\theta^*}} [V(h)] \right] \right| \nonumber \\
    &\le \mathbb{E}_{x \sim \rho} \left| \mathbb{E}_{h \sim P_{\hat{\theta}}} [V(h)] - \mathbb{E}_{h \sim P_{\theta^*}} [V(h)] \right| \nonumber \\
    &\le \mathbb{E}_{x \sim \rho} [ \sup_{H' \subseteq H} |P_{\hat{\theta}}(H') - P_{\theta^*}(H')| ] \nonumber \\
    &= \mathbb{E}_{x \sim \rho} [TV(P_{\theta^*} \| P_{\hat{\theta}})] \label{eq:tv_to_success}
\end{align}
Substituting the bound from Eq. \eqref{eq:tv_bound} into Eq. \eqref{eq:tv_to_success}, we finally obtain the bound on the difference in success rates:
$$
|\text{SuccessRate}(\hat{\theta}) - \text{SuccessRate}(\theta^*)| \le \mathcal{O}\left(\sqrt{\frac{d \log n + \log(1/\delta)}{n}}\right)
$$
Since we are concerned with the lower bound on the success rate, this is equivalent to:
\[
\text{SuccessRate}(\hat{\theta}) \ge \text{SuccessRate}(\theta^*) - \mathcal{O}\left(\sqrt{\frac{d \log n + \log(1/\delta)}{n}}\right)
\]
This result shows that, under curriculum training, the task success rate of the learned policy converges to that of the expert policy with high probability, at a rate of $\mathcal{O}(\sqrt{(d \log n)/n})$. This completes the proof.
\end{proof}

\end{document}